\documentclass{article}
\usepackage[utf8]{inputenc}
\usepackage[T1]{fontenc}
\usepackage[autolanguage,np]{numprint}
\usepackage[svgnames]{xcolor}
\usepackage{bm}
\usepackage{enumitem}
\usepackage{microtype}
\usepackage{amsmath}
\allowdisplaybreaks
\usepackage{amssymb}
\usepackage{amsthm}
\usepackage{booktabs}
\usepackage[strict]{csquotes}
\usepackage{graphicx}
\graphicspath{ {./assets/} }
\usepackage{caption}
\usepackage{subcaption}
\usepackage[square,sort,comma,numbers]{natbib}
\usepackage{mathtools}
\usepackage{url}
\usepackage{nicefrac}
\usepackage[]{todo}
\newcommand{\wik}{\textsc{Wikipedia}}
\newcommand{\sla}{\textsc{Slashdot}}
\newcommand{\pfc}{\textsc{Pfc}}
\newcommand{\epi}{\textsc{Epinion}}
\newcommand{\ssn}{signed social networks}
\newcommand{\eij}{\ensuremath{i\rightarrow j}}
\newcommand{\eji}{\ensuremath{j\rightarrow i}}
\DeclareMathOperator{\rank}{rank}

\newcommand{\by}{\boldsymbol{y}}
\newcommand{\hdop}{\widehat{d}_{\mathrm{out}}^+} 
\newcommand{\bE}{\mathbb{E}}
\newcommand{\yhat}{\widehat{y}}
\newcommand{\sgn}{\mbox{\sc sgn}}

\newtheorem{theorem}{Theorem}

\newtheorem{fact}{Fact}

\newcommand{\alclog}{\textsc{alclog}}
\newcommand{\alcone}{\textsc{alcone}}

\newcommand{\avgdeg}{\overline{d}}
\newcommand{\blc}{\textsc{blc}}
\newcommand{\dhat}{\widehat{d}}
\newcommand{\epinion}{\textsc{Epinions}}

\newcommand{\lge}{4\ceil*{\log (d_{\mathrm{out}}(i)+1)}}
\newcommand{\lgein}{4\ceil*{\log (d_{\mathrm{in}}(i)+1)}}
\newcommand{\lgeone}{\ceil*{\log (d_{\mathrm{out}}(i)+1)}}

\newcommand{\mhat}{\widehat{m}}

\newcommand{\slashdot}{\textsc{Slashdot}}
\newcommand{\that}{\widehat{t}}
\newcommand{\uhat}{\widehat{u}}

\renewcommand{\Pr}{\mathbb{P}}
\DeclarePairedDelimiter\ceil{\lceil}{\rceil}
\DeclarePairedDelimiter\floor{\lfloor}{\rfloor}

\newcommand{\scO}{\mathcal{O}}

\newcommand{\treestar}[1]{\textsc{Tree}$^\star$}

\newcommand{\comptriads}{16 Triads}
\newcommand{\complowrank}{LowRank}
\newcommand{\compasym}{AsymExp}
\newcommand{\comppp}{Perceptron}

\usepackage{times}
\usepackage{graphicx} 

\usepackage{algorithm}
\usepackage{algorithmic}

\usepackage[nohyperref,accepted]{icml2016} 
\definecolor{mydarkblue}{rgb}{0,0.08,0.45}
\usepackage{hyperref}
\hypersetup{%
    colorlinks=true, linktocpage=true, pdfstartpage=3, pdfstartview=FitV,%
    breaklinks=true, pdfpagemode=UseNone, pageanchor=true, pdfpagemode=UseOutlines,%
    plainpages=false, bookmarksnumbered, bookmarksopen=true, bookmarksopenlevel=1,%
    hypertexnames=true, pdfhighlight=/O,
    linkcolor=mydarkblue, citecolor=mydarkblue,
    filecolor=mydarkblue, urlcolor=mydarkblue,
    pdfauthor={Géraud Le Falher, Fabio Vitale},
    pdftitle={Even Trolls Are Useful: Efficient Link Classification in Signed Networks},
    pdfkeywords={social networks, signed graphs, negative links, link classification}
}
\icmltitlerunning{Active Link Classification in Signed Networks}

\begin{document}
\twocolumn[
\icmltitle{Even Trolls Are Useful: Efficient Link Classification in Signed Networks}

\icmlauthor{Géraud Le Falher }{geraud.le-falher@inria.fr}
\icmlauthor{Fabio Vitale}{fabio.vitale@inria.fr}
\icmladdress{Inria (Magnet Team), Univ. Lille, CNRS UMR 9189 -- CRIStAL,
  F-59000 Lille, France}


\icmlkeywords{social networks, signed graphs, negative links, link classification}

\vskip 0.3in
]

\begin{abstract} 
  We address the problem of classifying the links of signed social networks given their full structural topology. Motivated by a binary user behaviour assumption, which is supported by decades of research in psychology, we develop an efficient and surprisingly simple approach to solve this classification problem. Our methods operate both within the active and batch settings. We demonstrate that the algorithms we developed are extremely fast in both theoretical and practical terms. Within the active setting, we provide a new complexity measure and a rigorous analysis of our methods that hold for arbitrary signed networks. We validate our theoretical claims carrying out a set of experiments on three well known real-world datasets, showing that our methods outperform the competitors while being much faster.

\end{abstract}

\section{Introduction}
\label{sec:intro}

Connections in most social networks are driven by the {\em homophily} assumption, which can be described in the following way: linked individuals tend to be similar, sharing characteristics, attitudes or interests. However, homophily is not sufficient to explain the whole human behaviour in social networks. In fact, sociologists have also studied networks, hereafter called \emph{\ssn{}}, where {\em negative} relationships like dissimilarity, disapproval or distrust are explicitly displayed. Nowadays several online social networks present instances where the nature of a relationship can be negative. For instance, \textsc{Ebay}, where users trust or distrust agents in the network based on their personal interactions, \slashdot{}, where each user can tag another as friend or foes, and \epinion{}, where users can rate positively or negatively not only products, but also other users.

Many social networks are indeed rich in structure, involving heterogeneous data samples related to each other. This kind of data continue to gain attention especially in the online social network domain. One of the reasons underlying this interest is the possibility to classify comments in frequent online debates that occur between two users. Some of these comments can be viewed as negative links between the users involved, even in social networks where connections represent friendships solely.


Consider a user joining an online community. His behaviour will often fall into two cases. On one hand, the new member could play well with other users, establishing positive relationships, for instance with those who have been helpful. On the other hand, the new user could try to disrupt the community by engaging into the so called anti-social behaviour and creating conflictual relationships with other members.
In both cases, such binary behaviour assumption is supported by decades of research in psychology, starting with the seminal work of \citet{Dissonance57} about cognitive dissonance. This feeling of mental discomfort--when one is acting in contraction to one or more personal beliefs, ideas or values--, leads to user behaviour consistency.
A striking example of cognitive dissonance in economics is called \enquote{sunk cost} and describes the
tendency of individuals to knowingly make decisions resulting in a loss in
order to maintain consistency with their past behaviour~\cite{sunkCost85}.
On the relationships side, to justify the dissonance caused by doing something unpleasant, one tend to consider more valuable the causes and the results of his actions. For instance, when a user joins an online social community, he could interact anti-socially and later justify this behaviour by evaluating negatively the users involved in conflictual situations or debates about opposing opinions. In this case, this kind of attitude expressed publicly on social media leads to the definition of a \enquote{troll}:
\enquote{a user whose real intention(s) is/are to cause disruption and/or to trigger or exacerbate
conflict for the purposes of their own amusement.}
\cite{Hardaker10}.\citet{Shachaf10} elaborate on their motive, adding that
\enquote{boredom, attention seeking, and revenge motivate trolls; they regard
  Wikipedia as an entertainment venue, and find pleasure from causing damage to
the community and other people}.



Such signed social networks pose new challenges to the Machine Learning community. On one hand it calls for tailored methods to tackle existing problems, like user clustering, link prediction, targeted advertising, recommendation, prediction of user interests and analysis of the spreading of diseases in epidemiological models.
On the other hand, new problems emerge. Perhaps the most representative of these new problems is the classification of the connection nature that can be either positive or negative, when the whole network topology is known. In several situations, in fact, it is reasonable to assume that the discovery of the nature of a link is more costly than acquiring the topological information of the network. For instance, when two users of an online social network communicate on a public web page, it is immediate to detect the interaction. However, the classification of the interaction nature may require complex techniques. A crucial issue in developing algorithms for this kind of data is scalability, because of the huge amount of networked data. Hence, it is necessary to depart from well-established yet computationally expensive approaches by relying on novel, simple and ad-hoc algorithmic techniques.

In this work, we propose an extremely scalable link classification approach able to achieve a good accuracy while relying on a simple bias motivated by the above mentioned psychological considerations for social networks. More precisely, we consider directed \ssn{} and we posit that users are mostly consistent in their interactions, displaying either a positive or a
negative attitude. Such behaviour is indeed consistent with the cognitive dissonance theory
\cite{Dissonance57}.
This bias stems also from the intuition that most users
are part of the community because they enjoy interacting with some other
members, as shown by their positive relationships, whereas a minority of users is aggressive and obnoxiously looking for conflict, as shown by their outgoing negative links toward other members. As we said, this kind of users are commonly
referred as \emph{trolls} in online communities. This information relies only on edges viewed as outgoing from any given user.
As we show later however, a more accurate classification is obtained by combining the above described user troll feature with  edges viewed as ingoing to each users. More precisely, another user feature playing a crucial role is the pleasantness, since we noticed that in real world online datasets, most of the positive edges are ingoing to a relatively restricted subset of users. Hence, based on the fraction of negative outgoing edges from any user $i$, it is possible to estimate the {\em trollness} of $i$, i.e. to what extent $i$ is a troll user. In a symmetric way, based on the fraction of negative ingoing edges to any user $i$, it is possible to estimate the {\em unpleasantness}\footnote{For the sake of simplicity, in order to operate with two symmetric features, we focus on trollness and unpleasantness instead of trollness and pleasantness. The meaning of the unpleasantness feature can be viewed simply as the opposite of the one of the above mentioned pleasantness.}  of $i$, i.e. to what extent $i$ is an unpleasant user. These estimations are obtained analysing the observed sign of the edges present in the whole network, and provide a fast, accurate and simple approach to solve the Link Classification problem.

We present a set of algorithms based solely on these two features motivated by the related bias. We demonstrate that our algorithms are extremely fast both in theory and in practice, working within the active and batch settings. We provide a new complexity measure derived from our bias together with a rigorous analysis of our method performance within the active setting and we motivate our prediction rule based on the characteristics of the two features used. We validate our theoretical claims by carrying out a set of experiments on three different real-world datasets, showing that our methods perform better than the competitors while being much faster. Moreover, they are easy to implement.

Some advantages of our approach are determined by the fact it exploits features that can be computed locally, irrespective of the structural complexity of the network topology.
This locality have benefits not only in terms of interpretability but also
scalability, allowing us to envision applying our method to very large graphs. More precisely, the time to build our classifier is almost linear in the training set size, while the worst case time per prediction is constant.
Furthermore, this implies that computing the features could be parallelised in a distributed model where the nodes along with their incoming and outgoing edges are partitioned across machines.






\section{Related Work}\label{sec:related}
Interest in signed networks can be traced back to the psychological theory of
structural balance~\cite{cartwright1956structural,heider1958psychology} and
its weak version~\cite{davis1967clustering}.
The rise of online \ssn{} has allowed Computer Science to provide a more
thorough and quantified understanding of that phenomenon. Among several
approaches related to our work, some extend the spectral properties of a
graph to the signed case in order to find good embeddings for classification~\cite{Kunegis2009,SignedEmbedding15}. However, the use of the adjacency matrix
usually requires a total worst case time quadratic in the node set size, which prevent those methods to scale to large
graphs. Likewise, whereas the idea of mining ego networks with SVM provides good
results~\cite{Papaoikonomou2014}, the running time makes this approach often impractical for large real-world datasets.
A more local approach is provided by~\citet{Leskovec2010} with the so
called \enquote{status theory} in directed graphs. Some
works on active classification use a more sophisticated bias based on the Correlation Clustering problem~\cite{TreeStar12,CCCC12} provide strong theoretical guarantees, but the algorithm are necessarily more involved.
  
While we focus on binary prediction, it is possible to consider a weighted version of
the problem, in which case edges measure the amount of trust or distrust
between two users~\cite{guha2004propagation,tang2013exploiting,Bachi2012,Qian2014sn}.
Departing from our setting, link classification can also exploit side information
related to the network. For instance, \citet{EdgeSignsRating15} use the
product purchased on Epinion in conjunction with a neural network,
\citet{TrollDetection15} identify trolls by analysing the textual content of
their post and \citet{SNTransfer13} use SVM with transfer learning from one
network to another. While these approaches have interesting performance, they
require both extra information and time processing which prevent them to be as scalable as ours.

\section{Preliminaries and basic notation}\label{s:prel}

We focus on directed graphs without side information, where each edge can be
either positive or negative and is associated with a binary label in
$\{-1,+1\}$. Given a directed graph $G(V,E)$, where $V$ is the vertex set and
$E$ is the edge set, we denote the set of edges (ingoing and outgoing) incident
to a node $i \in V$ by $E(i)$. We also denote the subset of $E(i)$ formed by
the outgoing positive, outgoing negative, ingoing positive and ingoing negative
edges respectively by $E^+_{\mathrm{out}}(i)$, $E^-_{\mathrm{out}}(i)$, $E^+_{\mathrm{in}}(i)$ and
$E^-_{\mathrm{in}}(i)$. We define $E_{\mathrm{out}}(i) \triangleq E^+_{\mathrm{out}}(i) \cup E^-_{\mathrm{out}}(i)$
and $E_{\mathrm{in}}(i) \triangleq E^+_{\mathrm{in}}(i) \cup E^-_{\mathrm{in}}(i)$. 

Given any node $i \in V$, we define $d_{\mathrm{out}}^+(i) \triangleq \left|
E^+_{\mathrm{out}}(i) \right|$, which is therefore the number of outgoing positive edges
from $i$. Likewise, we define $d_{\mathrm{out}}^-(i) \triangleq \left| E^-_{\mathrm{out}}(i)
\right|$, $d_{\mathrm{in}}^+(i)\triangleq \left| E^+_{\mathrm{in}}(i) \right|$ and $d_{\mathrm{in}}^-(i)
\triangleq \left| E^-_{\mathrm{in}}(i) \right|$. We also define $d_{\mathrm{out}}(i) \triangleq
\left| E_{\mathrm{out}}(i) \right|$ and $d_{\mathrm{in}}(i) \triangleq \left| E_{\mathrm{in}}(i) \right|$.

A directed edge from node $i$ to node $j$ is represented by $i \rightarrow j$.
We use a binary vector $\by \in \{-1,+1\}^{|V|}$ to represent the labeling of each
edge $i \rightarrow j$. Hence, the label of edge $i \rightarrow j$ is denoted
by $y_{i \rightarrow j}$. If $y_{i \rightarrow j} = 1$ ($y_{i \rightarrow j} =
-1$), then node $i$ is providing a positive (negative) {\em signal} to node
$j$. We also define $y^{\mathrm{min}}_{\mathrm{out}}(i)$ and $y^{\mathrm{min}}_{\mathrm{in}}(i)$ as the
least\footnote{Ties are broken arbitrarily.} used label in $E_{\mathrm{out}}(i)$ and
$E_{\mathrm{in}}(i)$ respectively.

Finally, the {\em trollness} $t(i)$ of $i$ and his {\em unpleasantness} $u(i)$
are defined respectively as\footnote{In the special case $d_{\mathrm{out}}(i)=0$
($d_{\mathrm{in}}(i)=0$), we set $t(i)=\frac{1}{2}$ $\left(u(i)=\frac{1}{2}\right)$.}
$$t(i) \triangleq \frac{d_{\mathrm{out}}^-(i)}{d_{\mathrm{out}}(i)}\ \ \ \ \ \ \ \ \ \ \ \ \ \ \
\ u(i) \triangleq \frac{d_{\mathrm{in}}^-(i)}{d_{\mathrm{in}}(i)}$$ 

Trollness and unpleasantness are therefore two features associated with each
node $i \in V$ considering respectively only outgoing and ingoing edge labels.  

We use the \enquote{hat} symbol, together with the above described notation, to
identify the mentioned sets and quantities related to the subset of {\em observed} edges
solely. For example $\hdop(i)$ is the number of observed positive labels
associated with the edges of $E_{\mathrm{out}}(i)$.

\smallskip

We study the problem of transductive link classification in directed signed graphs, focusing especially on the active setting. However, the simple prediction rules we designed, directly derived from our bias, can be used even within the batch setting. The whole input graph topology is always completely known to the learner. 

We briefly recall the batch and active settings. In the batch setting, a subset of
the whole label set (the training set) is provided to the learner. The training set information is exploited to predict the remaining labels (test set). In the active setting, the learner is allowed to select the training set. The training set selection can be performed in such a way to minimise the number of prediction mistakes on the remaining labels\footnote{Our active setting is {\em non-adaptive}, i.e. the
learner is allowed to observe the labels of the training set {\em only after} the selection phase. The selection phase is therefore not driven by the revelation of the
labels queried (even if the labels can selected incrementally).}. 

\section{Link classification using trollness and unpleasantness}

Our transductive learning problem consists in the classification of the labels
of a subset of the edges of an input directed signed graph $G(V,E)$, given the whole
topology of $G$ when a subset of the edge labels is revealed. Since the input
graph has no side information, an algorithm for this problem can only exploit
the observed labels and the network topology. We design mainly two prediction rules based on only two node features which are defined in terms of the edge labels and the graph local topology. These rules are extremely simple yet very effective and attempt to estimate the value of the two node features we consider for all nodes in $V$. As we will explain in details in this section, the features associated with the nodes can also be viewed as features of the incident edges, taking into account the direction of each edge. 
Our prediction rules can be used within the transductive active and batch settings. We focus especially on the active setting, since the selection phase offers opportunities for
simple and principled methods to obtain a very good estimation
of these values. 

We consider signed social networks, like \sla{}, where users can tag other
users as friends or foes, \epi{} where users can rate other users positively or
negatively and \wik{}, where admins can vote for or against
the promotion of another admin. Our
approach for the link classification problem is based on the following
simple consideration: we notice that most of
the negative edges are outgoing from a restricted number of users. The labels of
the edges outgoing from these users are often negative. This suggests us--%
in accordance with psychological evidences of a consistent user behaviour
preventing cognitive dissonance--%
that a feature playing a crucial role in our prediction problem is the ratio
$\frac{d_{\mathrm{out}}^-(i)}{d_{\mathrm{out}}(i)}$, defined in Section~\ref{s:prel} as
the {\em trollness} $t(i)$ of any given node $i$. In a symmetric way, we notice
that usually the difference between the number of positive and negative labels
of the edges incoming in a vertex $i$ is always quite big with respect to the total number of edges incoming to $i$. Hence we considered a second feature defined as the ratio $\frac{d_{\mathrm{in}}^-(i)}{d_{\mathrm{in}}(i)}$ that we call {\em unpleasantness} $u(i)$ of $i$. Interestingly, only one of these two feature is sufficient to provide enough information to design a prediction rule extremely simple yet effective. The combination of the two features can be exploited even more efficiently. 

We then design a main simple method and a selection strategy for the active
setting, that turns out to be extremely fast, very effective in practice and
easy to implement. The results are easy to interpret and for the active setting, in order to provide insight into the
use of these basic features, we analysed in details one of our methods (and give proofs in \autoref{ssec:proofs}).
Finally, these methods are very local, which
entails they can be even easily parallelizable. 

\subsection{Predicting using only one feature}\label{s:one_feature}

The first step of our method using only the trollness feature is go through the
training set and assign
to each node $i \in V$ the value $\that(i)$ equal to the ratio
$\frac{\dhat_{\mathrm{out}}^-(i)}{\dhat_{\mathrm{out}}(i)}$ between the number of observed
negative labels and the total number of observed labels (positive and negative)
assigned to the edges outgoing from $i$. Analogously, using the unpleasantness
feature we assign the value
$\uhat(i)\triangleq\frac{\dhat_{\mathrm{in}}^-(i)}{\dhat_{\mathrm{in}}(i)}$ with each node $i \in V$. 

In order to motivate the choice of our simple prediction rule described below, we now provide a key fact about the optimal threshold maximizing the number of training edge corrected classified using only one feature. 


\begin{fact}\label{t:thresold_one_dimension} Given any training set for any labeled signed directed graph, the threshold $\frac{1}{2}$ maximises the number of  training edges classified with the correct labels using only the trollness or the unpleasantness feature.
\end{fact}

Hence, our simplest prediction rule based on a single feature is\footnote{Ties are broken arbitrarily, i.e., whenever $\that(i)$ or
$\uhat(j)$ are equal to $\frac{1}{2}$, label $y_{i \rightarrow j}$ can be
predicted with a label arbitrarily chosen in preliminary phase.}\vspace{-0.6\baselineskip} $$\yhat_{i \rightarrow
j}=\sgn\left(\frac{1}{2} - \that(i) \right)\quad\quad
\yhat_{i \rightarrow j} = \sgn\left(\frac{1}{2} - \uhat(j)\right)$$ 
using trollness and unpleasantness respectively. 

We predict therefore $\yhat_{i \rightarrow j}=1$ with the
majority of the labels of the outgoing edges from $i$ or ingoing ones to $j$ simply with a majority vote rule\footnote{Ties are broken arbitrarily.}, exploiting locally the estimated trollness or unpleasantness of each node $i$ or $j$. We call $\blc(t)$ (Batch Link
Classifier) and $\blc(u)$ the algorithms using this prediction rule within the
batch setting, respectively using the trollness or unpleasantness solely. Finally, when a single feature is used, it can be seen as associated with node $i$ as well as with any edge respectively outgoing from $i$ (using~$t(i)$) or ingoing to $i$ (using~$u(i)$). 

\subsection{Predicting combining trollness and
unpleasantness}\label{s:two_features}

When using both trollness and unpleasantness, each edge $i \rightarrow j$ can
be viewed as a point with coordinates $\that_{\mathrm{out}}(i)$ and $\uhat_{\mathrm{in}}(j)$ in
the compact set $[0,1]\times[0,1]$. A natural extension of the prediction rule
using a single feature explained in \autoref{s:one_feature} could be the
following\footnote{Ties are broken arbitrarily in this case too, i.e. whenever
$\that_{\mathrm{out}}(i)+\uhat_{\mathrm{in}}(j)-1$ is equal to $0$, label $y_{i \rightarrow j}$
can be predicted with a label arbitrarily chosen in a preliminary phase.}:
$$\yhat_{i \rightarrow j}=\sgn\left(\that_{\mathrm{out}}(i)+\uhat_{\mathrm{in}}(j)-1\right) $$
which is interpreted  saying that a label is predicted $-1$ if and only if the estimated trollness $\that_{\mathrm{out}}(i)$ of node $i$ is larger than
the estimated \enquote{pleasantness} $(1-\uhat_{\mathrm{in}}(j))$ of node $j$.
This is equivalent to using linear separator in the plane with
coordinates $\that_{\mathrm{out}}(i)$ and $\uhat_{\mathrm{in}}(j)$. This simple prediction rule
takes into account the symmetry of the two features, but it is not guaranteed
to be the optimal linear separator for the training set in the compact set
$[0,1]\times[0,1]$.
Exploiting these two features for several datasets, linear classifiers such as
the perceptron algorithm return a linear separator whose equation is very close to
$\that_{\mathrm{out}}(i)+\uhat_{\mathrm{in}}(j)-k=0$ for a certain value of $k$ depending on the training set. This suggest that a good linear separator using these two features is a line parallel to $\that_{\mathrm{out}}(i)+\uhat_{\mathrm{in}}(j)-1=0$, and that one parameter ($k$ in $\that_{\mathrm{out}}(i)+\uhat_{\mathrm{in}}(j)-k=0$) needs to be learned.
Indeed, using the rule
$$\yhat_{i \rightarrow j}=\sgn\left(\that_{\mathrm{out}}(i)+\uhat_{\mathrm{in}}(j)-k^*\right)$$
where $k^*$ is the optimal\footnote{We refer to the separator maximising the number of training edges classified in a correct way.} $k$'s value for separating the training points on the
plane with a line expressed by the equation $\that_{\mathrm{out}}(i)+\uhat_{\mathrm{in}}(j)-k=0$, we are able to obtain very good results while the resulting algorithm is extremely simple and fast.
More precisely finding $k^*$ requires a time equal to
$\scO(|E_{\mathrm{train}}|\log|E_{\mathrm{train}}|)$ using the following method (we recall that
each point on the compact set $[0,1]\times[0,1]$ corresponds to a training
edge): 
\begin{enumerate}[nosep,leftmargin=*]
\item For each training label $y_{i \rightarrow j}$ we
    compute the distance between the  line defined by the equation
    $\that_{\mathrm{out}}(i)+\uhat_{\mathrm{in}}(j)=0$ and the point representing edge $i
  \rightarrow j$.  
  \item We create a sequence $S$ containing all the training
  points sorted by the distance computed in the previous step.   
  \item We scan the whole sorted sequence $S$ to find the optimal separator $S_{(i\rightarrow j)^*}$ for the given training set and we finally set
  $k^*=\frac{\sqrt{2}}{2}S_{(i\rightarrow j)^*}$.  
  \end{enumerate}

Hence $\that_{\mathrm{out}}(i)+\uhat_{\mathrm{in}}(j)-k^*=0$ is the optimal separator line parallel to $\that_{\mathrm{out}}(i)+\uhat_{\mathrm{in}}(j)-1=0$ for the
given training set. In
\autoref{f:figexpl} we illustrate this simple algorithm. We call this method $\blc(t,u)$.

\begin{figure}[t!]\centering
\includegraphics[width=50mm]{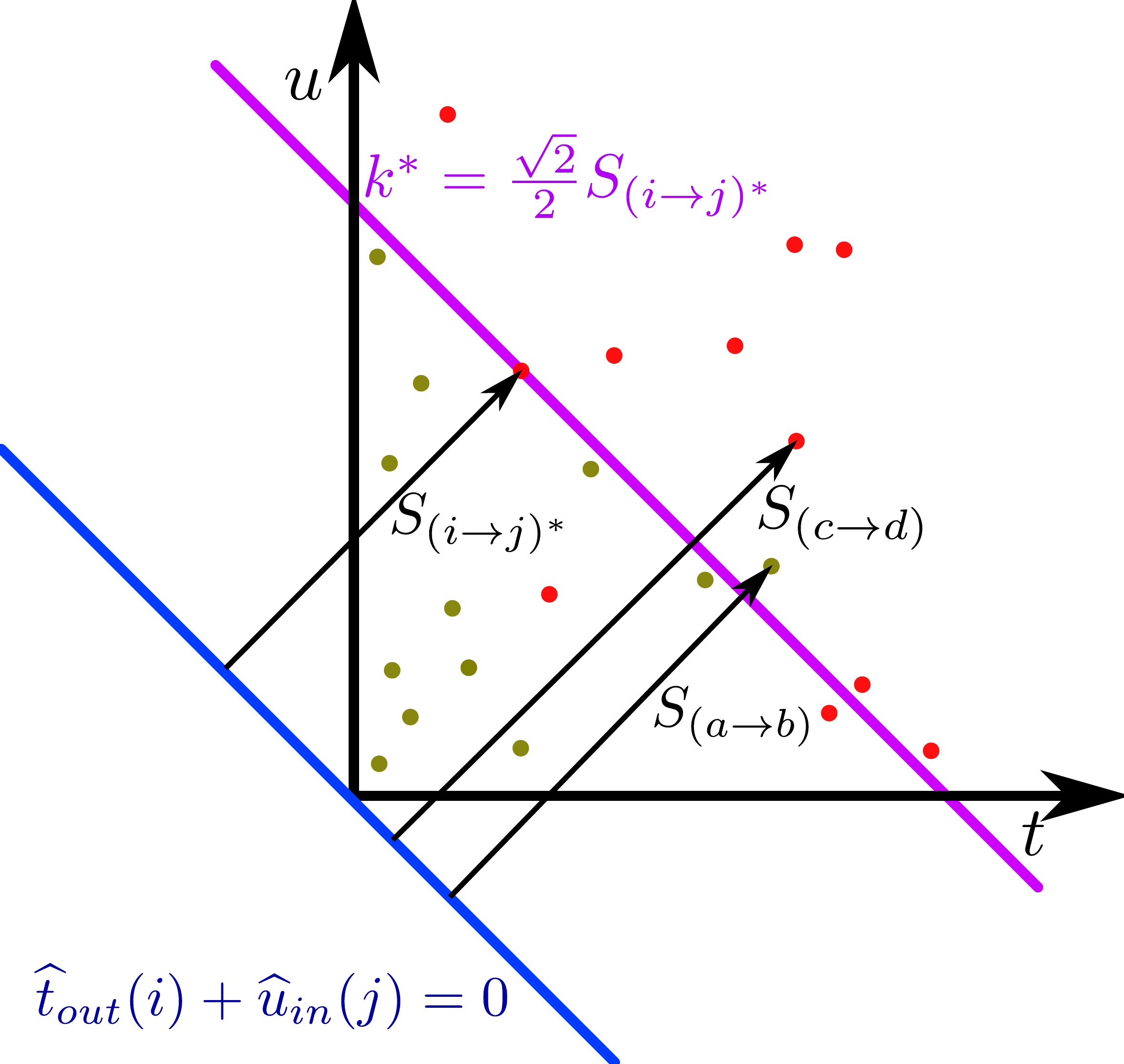}
\caption{Combining trollness and unpleasantness to find the best linear separator parallel to the line defined by the equation $\that_{\mathrm{out}}(i)+\uhat_{\mathrm{in}}(j)-1=0$ (or, equivalently, $\that_{\mathrm{out}}(i)+\uhat_{\mathrm{in}}(j)=0$).\label{f:figexpl}}
\vspace{-0.7cm}
\end{figure}

\subsection{Active learning algorithms}

Since the estimation of the true trollness $t(i)$ and unpleasantness $u(i)$ for
each node $i \in V$ plays a crucial role in these prediction rules,
we devise two extremely fast and simple selection strategy and we analyse
them in details in \autoref{s:analysis}: $\alcone$ and $\alclog$ (Active Link
Classifier using,  to estimate each node feature, {\em one} label or a {\em logarithmic} number of labels). Both strategies are very scalable in that they
require only a worst case constant time per query. Analogously to the $\blc$
algorithm, $\alcone(t)$ and $\alclog(t)$ use only the trollness feature whereas
$\alcone(u)$ and $\alclog(u)$ use only the unpleasantness feature. $\alcone(t,u)$ and $\alclog(t, u)$ use instead both the feature. 

The prediction rule is all cases is the same described in
\autoref{s:one_feature} using only one feature and in \autoref{s:two_features}
using two features. 

In the selection phase, for all node $i \in V$, $\alcone(t)$ and $\alcone(u)$
pick one edge uniformly at random in $E_{\mathrm{out}}(i)$ or $E_{\mathrm{in}}(i)$ respectively (if it is not an empty set) and
query its label. This query will be then used to predict all the other labels
of the edges in $E_{\mathrm{out}}(i)$ or $E_{\mathrm{in}}(i)$ respectively. $\alclog(t)$ and
$\alclog(u)$ work in an analogous way, selecting uniformly at random with
replacement $\lge$ and $\lgein$ edges respectively, and asking for their labels. This
can provide, as demonstrated in \autoref{s:analysis}, better guarantees on
the number of mistakes occurring during the prediction phase. 

Finally, for each node $i \in V$, $\alcone(t, u)$ picks one edge uniformly at
random both from $E_{\mathrm{out}}(i)$ and in $E_{\mathrm{in}}(i)$ (if they are not empty sets). $\alclog(t, u)$ operates analogously, picking uniformly at random with replacement $\lge$ edges from $E_{\mathrm{out}}(i)$ and $\lgein$ from $E_{\mathrm{in}}(i)$. Both $\alcone(t, u)$ and $\alclog(t, u)$, in the last step of the selection phase, query the labels of all the selected edges.

\section{Link classification analysis}\label{s:analysis}

\subsection{Complexity measures} \label{sub:complexity_measures}

We begin defining the complexity measures that we use to characterise the
labeling for this problem. We basically have three complexity measures:
$\Psi_{\mathrm{out}}$, $\Psi_{\mathrm{in}}$ and $\Psi_{t,u}$, which refer respectively to the
trollness, the unpleasantness and to the combined use of both the features.
$\Psi_{\mathrm{out}}$ is defined as the sum over all nodes $i$ of the number of outgoing edges labeled with the least used labels in $E_{\mathrm{out}}(i)$. The definition of $\Psi_{\mathrm{in}}$ is equal to the one of $\Psi_{\mathrm{out}}$ except for the fact of considering only the ingoing edges, i.e. the unpleasantness feature.
$\Psi_{t,u}$ takes into account both the features
but we will
not enter into the details
since focusing on $\Psi_{\mathrm{out}}$ and $\Psi_{\mathrm{in}}$ provides the most significant information about the theoretical motivations underlying our learning approach. Finally, since these measures depend on the labeling, in some cases we will make explicit this dependence using notations like $\Psi_{\mathrm{out}}(\by)$. 

We will now give an analysis of our simplest decision rule, based on a single
feature and the natural threshold $\frac{1}{2}$, proven to be optimal for separating the training set in \autoref{s:one_feature}. Hence, a node $i$ is considered a troll if the majority of the outgoing edges from $i$ are negative, and unpleasant if the majority associated with the incoming edges to $i$ are negative\footnote{Ties are broken arbitrarily, i.e., for any given node $i$,
if $d_{\mathrm{out}}^-(i) = d_{\mathrm{out}}^+(i)$, then $i$ can be
equivalently considered a troll or a non-troll node. If $d_{\mathrm{in}}^-(i) = d_{\mathrm{in}}^+(i)$, then $i$ can be equivalently considered unpleasant or pleasant node.}. 

For the sake of simplicity, hereafter in this section, we will use $\Psi$ instead of $\Psi_{\mathrm{out}}$. The
label irregularity measure $\Psi$ can be defined formally as:
$$\Psi \triangleq \sum_{i \in V} \min\bigl(d^+_{\mathrm{out}}(i),
d^-_{\mathrm{out}}(i)\bigr)$$
We also define the labeling irregularity contribution of each node $i$ as
$\Psi(i) \triangleq  \min\bigl(d^+_{\mathrm{out}}(i), d^-_{\mathrm{out}}(i)\bigr)$, which measures to what extent the behaviour of a node is self-consistent
according to our bias, i.e. how many positive edges are outgoing from a troll
node and how many negative edges are outgoing from a non-troll node.

\subsection{Active setting}

The goal of this analysis is to provide insight into the use of these basic
features within the active setting and, at the same time, to show the solidity and the reliability of our approach from a theoretical perspective. Since the analysis for the trollness and unpleasantness are basically equal, without loss of generality we will focus on the trollness feature and the complexity measure
$\Psi_{\mathrm{out}}$. Moreover, given any algorithm $A$, we denote the number of mistakes
made by $A$ by $m_A$.

In the active setting, we assume that all labels are set {\em before} the
algorithm selection phase. Hence, according to our bias, in the selection phase
the learner needs at least one edge from each node $i$ in order to predict the
labels of the edges outgoing from $i$. We provide now an interesting and general lower bound for the number of mistakes that can occur within the active setting. 

\begin{theorem}\label{t:active_lb} Given any directed signed graph G(V, E) and
  labeling irregularity budget $K \le \floor*{\frac{|E|}{2}}$, there exist a
  randomized label strategy $\by$ such that the
  expected number of mistakes made by any active algorithm $A$ asking any
  number $Q$ of queries satisfies $$\bE m_A \ge \frac{\Psi(\by)}{|E|}\bigl(|E|-Q\bigr)-1$$ 
  while $\Psi(\by) \le K$.
\end{theorem}

We now prove that $\alcone$ always
makes in expectation  a number of mistakes close to our active lower bound
expression in Theorem\ \ref{t:active_lb} when the difference between $|E|$ and the number of queries $Q$ is not too small.

\begin{theorem}\label{t:active_one} 
Given any directed signed graph $G(V, E)$
  and any labeling $\by$, the expected number of mistakes made by $\alcone$
  satisfies $$\bE m_{\alcone} \le 2\Psi(\by)$$ while the number of
  distinct labels queried is not larger than $|V|$.
\end{theorem}



Finally, we conclude this section analysing $\alclog$ and quantifying to what extent the accuracy performance guarantees can improve using this approach when we are allowed to query a larger number of labels. 

Let $\avgdeg$ be the average node degree in $G$, i.e. $\avgdeg=\frac{|E|}{|V|}$ and let $\overline{\Psi}_{0}(\by)$ be the average value of $\Psi(i)$ over all nodes $i$ such that $\Psi(i)\neq0$.

\begin{theorem}\label{t:active_log} 
Given any directed signed graph G(V, E)
  and any labeling $\by$ such that $\Psi(\by)>0$, the expected number of mistakes made by $\alclog$
  satisfies 
  $$\bE m_{\alclog} = \Psi(\by)+\mathcal{O}\left({\frac{\Psi(\by)}{\sqrt{\log\left(4\overline{\Psi}_{0}(\by)+1\right)}}}\right)$$ 
  while the number of distinct labels queried is equal to
  $\mathcal{O}\left(|V|\log\left(\avgdeg+1\right)\right)$ when $|E|=\Omega(|V|)$. 
\end{theorem}

Theorem\ \ref{t:active_log} quantifies the performance of $\alclog$ in the worst case labeling for arbitrary input graphs.
We conjecture that the mistake expression of Theorem\ \ref{t:active_log} can be dramatically reduced in a non-adversarial setting, where the labels are assigned to edges in a stochastic way. An example of such setting could be the following. In a given arbitrary input graph, starting from a labeling $\by'$ such that $\Psi(\by')=0$, each label is flipped independently with a probability $p \le 1/2$. Hence, in this setting $\bE \Psi(\by) \le p|E|$ and $p$ is therefore a parameter expressing to what extent labeling $\by$ can be irregular with respect to our bias. Although the worst case analysis is important in that it provides strong performance guarantees holding for arbitrary labeling, we believe that an analysis for the above described non-adversarial setting can offer significant information about the performance of our methods on real-world signed networks.

  



\section{Experiments}\label{sec:results}
In this section, we evaluate our Link Classification method on representative real world
datasets within two different transductive learning settings, batch and active,
showing that it competes well against existing methods in terms of predictive and computational performance.

These three directed \ssn{} are widely
used as benchmark for this task. Specifically, in \wik{}, there is an
edge from user $i$ to user $j$ if $j$ applies for an admin position and $i$ votes
for or against that promotion. In \sla{}, a news sharing and
commenting website, members $i$ can tag other
members $j$ as friends or foes. Finally, in
\epi{,} an online shopping website, each user $j$ can review
products. Based on these reviews, another user $i$ can display whether
he considers $j$ to be reliable or not. Although the
semantic of the signs is not the same in these three networks, we will
see that our bias apply to all of them. From the datasets properties
showed in \autoref{tab:dataset}, we note that most edge labels are positive.
Therefore, to assess our prediction performance, in addition to the standard
measure of accuracy, we also report the Matthews Correlation Coefficient~\citep[MCC]{MCC00},
defined as:  \[
	\mathrm{MCC} = \frac{tp\times tn-fp\times fn} {\sqrt{ (tp + fp) ( tp + fn ) ( tn + fp ) ( tn + fn ) } }
\]
It combines all the four quantities found in a binary confusion
matrix ($t$rue $p$ositive, $t$rue $n$egative, $f$alse $p$ositive and $f$alse $n$egative)
into a single metric and ranges from $+1$ (when all predictions are
correct) to $-1$ (when all predictions are incorrect) through $0$ (when
prediction are made uniformly at random).
Another notable feature of our datasets
is the relative regularity of their labeling with respect to our
complexity measures defined in \autoref{s:analysis} (especially in \epi{} case, which is confirmed by higher accuracy).

\begin{table}[b]
  \centering
  \small
  \caption{Dataset properties.\label{tab:dataset}}
  \begin{tabular}{lrrrrrrrr}
    \toprule
    {}     & $|V|$       & $|E|$       & $\frac{|E^+|}{|E|}$ & $\frac{\Psi_{\mathrm{in}}}{|E|}$ & $\frac{\Psi_{\mathrm{out}}}{|E|}$\\
    \midrule
    \wik{} & \np{7115}   & \np{103108} & 78.79\%               & 0.19                    & 0.14            \\
    \sla{} & \np{82140}  & \np{549202} & 77.40\%               & 0.17                    & 0.14            \\
    \epi{} & \np{131580} & \np{840799} & 85.29\%               & 0.07                    & 0.09            \\
    \bottomrule
  \vspace{-0.7cm}
  \end{tabular}
\end{table}


We evaluate the prediction rules described in
Sections~\ref{s:one_feature} and \ref{s:two_features}.
However, we report only the use of \emph{trollness} as a single feature, as it
performs as well as \emph{unpleasantness} on \epi{} and has 5\% higher accuracy
on \wik{} and \sla{}.
Moreover, we use our estimated features to train two well
established linear separators as implemented by the Scikit-learn
library~\cite{scikit-learn}: Logistic Regression and Perceptron.
Yet, we report only Perceptron results since both methods have similar performances.
Finally, we experiment with the Nyström approximation of a Gaussian kernel~\cite{Nystrom}
but concluded that the small performance gain was not worth the
large increase in training time.

In the active setting, we compare our approach with the \treestar{k} (\enquote{TreeletStar} in~\cite{TreeStar12})
algorithm, which assumes that \enquote{edge labels are obtained through
perturbations of an initial sign assignment consistent with a two-clustering
of the nodes}. It operates by first
connecting stars of short subtrees extracted from a spanning tree and
querying the sign of the edges forming this spanning structure. Then it predicts the sign of
the other edges $(i, j)$ based on a label multiplicative rule used in a path connecting $i$ with $j$ for which all its labels are revealed. 
Although this method disregards the direction of the edges, to the best of our
knowledge it is the only one which solves the Link Classification problem in the
active setting.

\begin{figure*}[!t]
	\centering
	\begin{subfigure}[b]{.32\textwidth}
		\centering
		\includegraphics[height=.189\textheight]{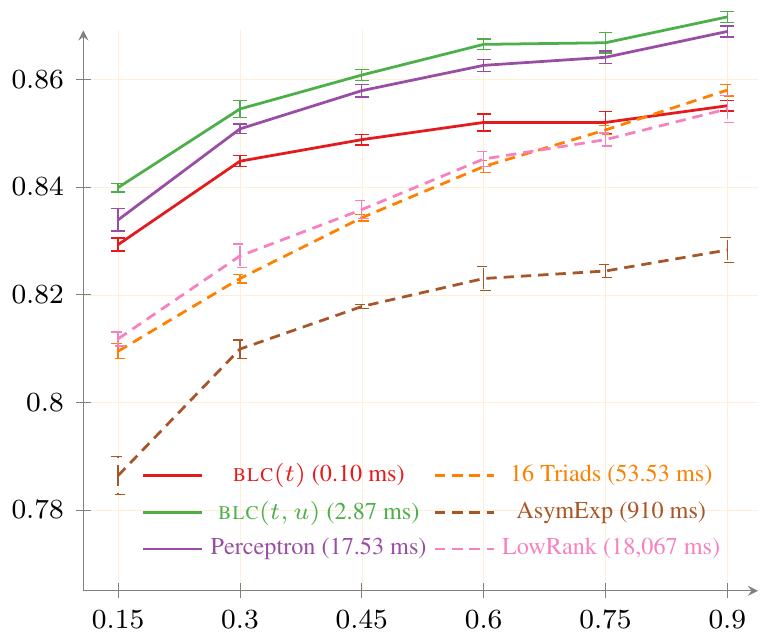}
		\caption{\wik{} Accuracy \label{fig:wik_batch_acc}}
	\end{subfigure}
	\begin{subfigure}[b]{.32\textwidth}
		\centering
		\includegraphics[height=.189\textheight]{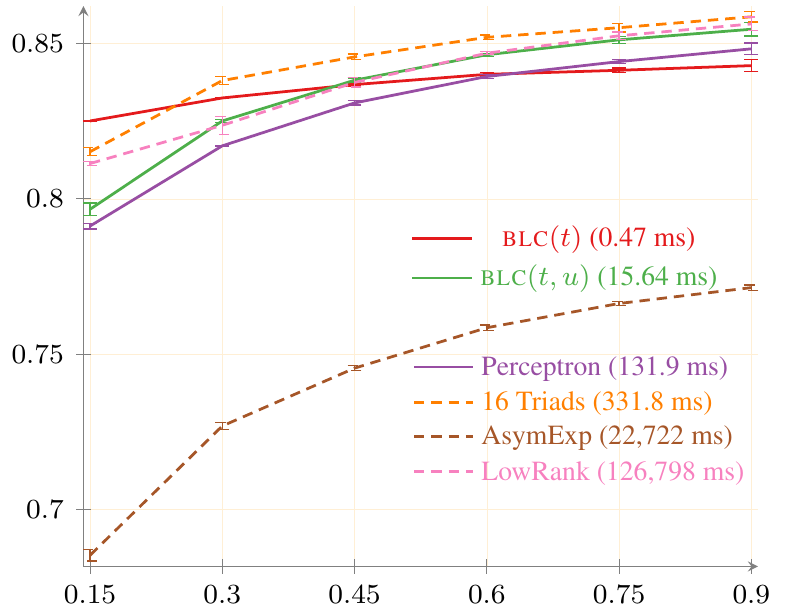}
		\caption{\sla{} Accuracy \label{fig:sla_batch_acc}}
	\end{subfigure}
	\begin{subfigure}[b]{.32\textwidth}
		\centering
		\includegraphics[height=.189\textheight]{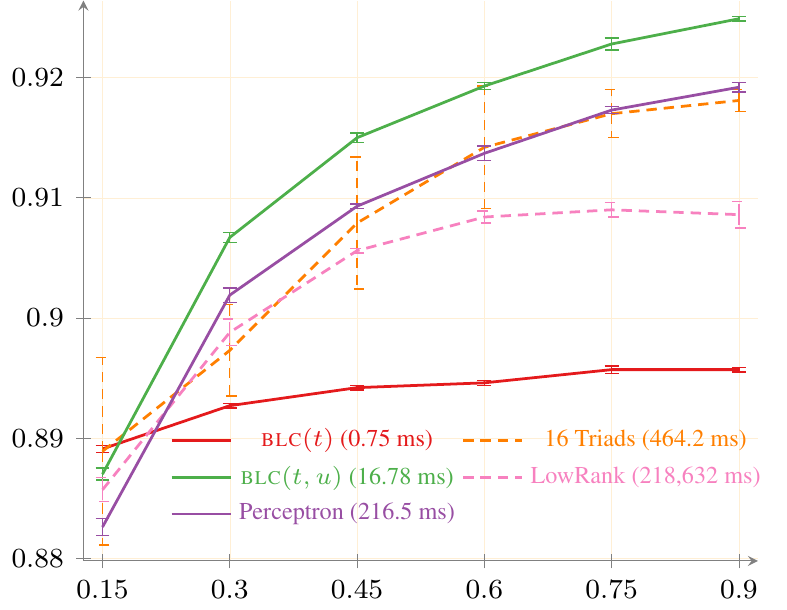}
		\caption{\epi{} Accuracy \label{fig:epi_batch_acc}}
	\end{subfigure}
	~
	\begin{subfigure}[b]{0.32\textwidth}
		\centering
		\includegraphics[height=.189\textheight]{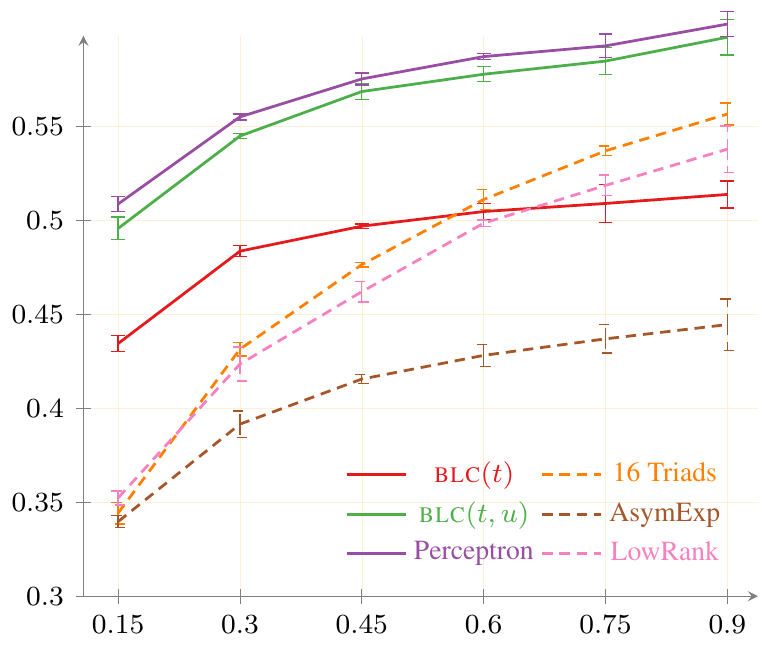}
		\caption{\wik{} MCC \label{fig:wik_batch_mcc}}
	\end{subfigure}
	\begin{subfigure}[b]{.32\textwidth}
		\centering
		\includegraphics[height=.189\textheight]{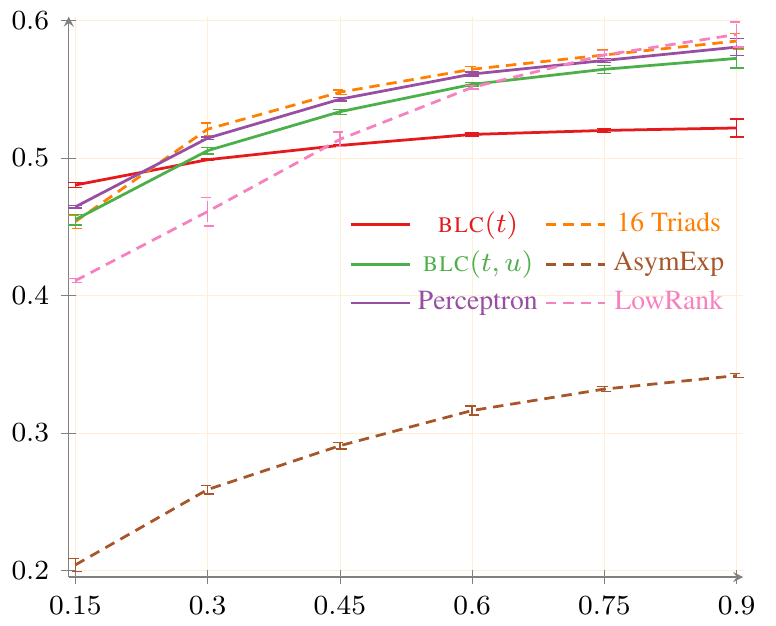}
		\caption{\sla{} MCC \label{fig:sla_batch_mcc}}
	\end{subfigure}
	\begin{subfigure}[b]{.32\textwidth}
		\centering
		\includegraphics[height=.189\textheight]{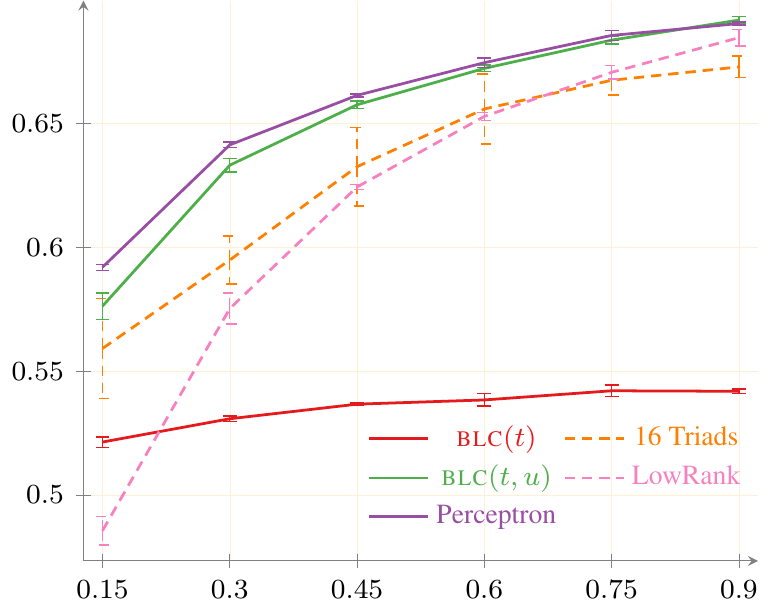}
		\caption{\epi{} MCC \label{fig:epi_batch_mcc}}
	\end{subfigure}
	\caption{Experimental results: In addition to accuracy, the first row
	  indicates the time needed by each method to process the training
	  set and make its prediction. Errors bars denote one standard
	  deviation around the mean of the 12 runs. We do not include AsymExp
	  in \epi{} because its performance was far from all other methods. Compared with previous
	works in dashed lines, our methods compete very well while being much faster.\label{fig:batch_res}}
\end{figure*}

In the batch setting, we implemented three methods covering
three different paradigms: local status theory, global matrix completion and spectral
decomposition.\\
The status theory heuristic posits that a positive link from
    $i$ to $j$ denotes that user $i$ considers user $j$ as having a higher status or skill
    than himself~\cite{Leskovec2010}. It implies that among the 16 possible
    triads of three nodes with two directed signed edges among them, some must
    be more represented than others. The \emph{16 Triads} method exploits this fact by counting
    for each edge in the training set how frequently it is involved in each of
    the 16 triad types.
    Based on these edge features and 7 degree features of each edge endpoints,
    a Logistic Regression model is trained and
    use to predict the sign of the test set edges.
    \\
Looking at undirected graphs with a more global perspective,
\citet{LowRankCompletion14} consider the observed adjacency matrix $A$, made of
the edges in $E_{\mathrm{train}}$, as a noisy sampling of the adjacency matrix $A^\star$ of an
underlying complete graph satisfying the weakly balance condition (that is with
no triangle containing only one negative edge). This
condition implies the existence of a small number $k$ of node clusters with positive
links within themselves and negative links across clusters, which in turn
yields $\rank(A^\star)=k$.
By recovering a complete matrix $\tilde{A}$ which matches the non-zeros entries of $A$,
it is possible to predict the sign of $\eij \notin E_{\mathrm{train}}$ as $\yhat_{i,j} =
\sgn(\tilde{A}_{i,j})$. Although the exact version of this problem is NP-hard,
authors assume that $k$ is an hyperparameter known beforehand\footnote{The
method is showed to be robust to the choice of its value, thus we pick $k=7$.}
and look for two matrices $W,H\in \mathbb{R}^{k\times|V|}$ which minimise a sigmoidal
reconstruction loss with $A$, subject to a regularization term. We refer to
this method as \emph{LowRank}.
\\
  Finally, the \emph{AsymExp} heuristic~\cite{Kunegis2009} computes the
  exponential of the observed symmetrized adjacency matrix $A+A^T$ after it has been reduced to $z$
  dimensions by eigenvector decomposition. This allows summing the product of the signs of all paths between two
  nodes with decreasing weight depending of the path length. The label of a
  test edge is predicted as the sign of this sum. We set the
  parameter $z$ equal to $15$ because it is one of the best in~\citep[Fig.
  11]{Kunegis2009} and it performs well on real datasets in~\citep[Fig. 3]{TreeStar12}.
  Although the authors also present a version using the directed adjacency matrix $A$, it
  performed slightly worse in our testing.


We start presenting our results within the active setting, as it follows more closely our theoretical
analysis. Moreover, it has a natural application in social networks: by
choosing a small set of edges and asking users to label them, it is possible to
accurately predict the sign of the other edges of the network. This is valuable even outside
the \ssn{} context, since regular social networks like Facebook also present
negative relationships, albeit not explicitly~\cite{Frenemy12}.
To provide a fair comparison with \treestar{k}, we restrict our
sampling to the largest weakly connected component, which however
includes $99.9\%$, $100\%$ and $99.1\%$ of edges in \wik{}, \sla{} and \epi{}
respectively.
We let each method performs its sampling of edges 12 times and report averaged
results in \autoref{tab:active}. It is striking that by querying far few edges
than a spanning tree and running in less than a millisecond, \alcone(t) manages
to reach almost 80\% accuracy over all the other remaining edges. Although
querying more edges with \alclog$(t)$ do not markedly improve performances,
combining a larger budget with the use of our two features learning procedure \alclog$(t,
u)$ outperforms existing methods with a low running time matching the algorithm complexity
presented in \autoref{s:two_features}.

\begin{table}\centering
  \setlength{\tabcolsep}{4.5pt}
  \small
  \caption{Mean Accuracy and MCC (with their standard
    deviation) of active algorithms.
    We also report the mean fraction of
  edges queried and the time taken to query the labels and make the prediction
(but no variance as they are very close to zero).\label{tab:active}}
  \begin{tabular}{lrrrr}
    \toprule
  {}              & Accuracy       & MCC             & $\frac{|E_{\mathrm{train}}|}{|E|}$ & time      \\
    \midrule
    & \multicolumn{4}{c}{\wik{}} \\
    \cmidrule{2-5}
    \alcone$(t)$    &  $79.8 \pm 0.4$  &  $38.5 \pm 1.0$  &   $2.3\%$  &  $0.2$ ms \\
    \alclog$(t)$    &  $84.5 \pm 0.1$  &  $43.3 \pm 0.5$  &  $27.6\%$  &  $0.1$ ms \\
    \alcone$(t, u)$ &  $79.6 \pm 0.4$  &  $39.4 \pm 0.7$  &   $8.0\%$  &  $0.7$ ms \\
    \alclog$(t, u)$ &  $87.2 \pm 0.1$  &  $54.6 \pm 0.6$  &  $48.7\%$  &  $2.5$ ms \\
    \treestar{7}    & $69.3 \pm 2.0$ & $20.1 \pm 2.6$  & $7.0\%$   & $253$ ms \\
    \midrule
    & \multicolumn{4}{c}{\sla{}} \\
    \cmidrule{2-5}
    \alcone$(t)$    &  $78.6 \pm 0.3$  &  $40.4 \pm 0.7$  &   $8.0\%$  &   $0.8$ ms \\
    \alclog$(t)$    &  $83.3 \pm 0.1$  &  $53.4 \pm 0.2$  &  $35.4\%$  &   $0.6$ ms \\
    \alcone$(t, u)$ &  $81.5 \pm 0.1$  &  $46.0 \pm 0.7$  &  $19.0\%$  &   $4.5$ ms \\
    \alclog$(t, u)$ &  $86.0 \pm 0.0$  &  $59.1 \pm 0.3$  &  $60.8\%$  &  $15.8$ ms \\
    \treestar{8}    & $69.5 \pm 0.6$ & $24.4 \pm 0.8$  & $16.4\%$   & $4685$ ms \\
    \midrule
    & \multicolumn{4}{c}{\epi{}} \\
    \cmidrule{2-5}
    \alcone$(t)$    &  $86.2 \pm 0.2$  &  $46.5 \pm 0.8$  &  $10.6\%$  &   $1.2$ ms \\
    \alclog$(t)$    &  $89.1 \pm 0.1$  &  $58.4 \pm 0.2$  &  $32.9\%$  &   $1.0$ ms \\
    \alcone$(t, u)$ &  $89.8 \pm 0.1$  &  $59.8 \pm 0.4$  &  $17.8\%$  &   $5.8$ ms \\
    \alclog$(t, u)$ &  $93.8 \pm 0.1$  &  $72.1 \pm 0.2$  &  $52.5\%$  &  $14.9$ ms \\
    \treestar{9}    & $74.6 \pm 2.4$ & $28.8 \pm 3.7$  & $16.9\%$   & $7859$ ms \\
    \bottomrule
  \end{tabular}
\end{table}

Because most methods we compared with operate within the batch setting, we also evaluate our
approach in that setting. Specifically, we sample uniformly at random
$15$, $30$, $45$, $60$, $75$ and $90$\% of the edges and train each
method on that sample before predicting the edges not sampled. This procedure is repeated 12 times to account for the randomness of the sampling.
Although some competitors are able to outperform our \blc$(t)$ rule given a
training set large enough, the use of our two features still let us achieve the overall
best results in the shortest amount of time. An exception to that statement is
the accuracy on the \sla{} dataset. However, a simple heuristic allows us to
regain our crown: when predicting $\eij \in E_{\mathrm{test}}$, if the reciprocal edge
\eji{} is part of the training set, we set $\sgn(\eij) = \sgn(\eji)$. Because
most reciprocal edges have the same sign, this gives Accuracy and MCC a boost
of 0.4 to almost 4\% depending on the dataset (see
Tables~\ref{tab:batch_postprocessing} and \ref{tab:batch_postprocessing_mcc}
for full details).
Another point to notice is that with an equal budget, e.g. 60\% of \epi{} edges,
the sampling strategy \alclog$(t,u)$ achieves $93.6\%$ accuracy whereas
\blc$(t,u)$ can only reach $92.5\%$, highlighting the benefits of our careful
active selection of training edges.

\section{Conclusion}\label{sec:conclusion}

In this work, we showed a new approach to solve the transductive Link Classification problem operating within the active and batch setting. We demonstrated our methods are able to achieve very good performance while being extremely fast. We also provided a rigorous analysis of our main approach within the active setting, based on a very natural and reasonable complexity measure for this problem. Since our algorithms exploit two features that are local irrespective of the structural complexity of the input network topology, we believe our approach can be easily parallelised. Furthermore, we plan to apply our methods to users-products bipartite input graphs, where users provide positive or negative feedbacks for products and a central issue is to predict the user preferences. Finally, we believe our model could be easily extended to the online learning setting and our approach could be modified in order to operate with weighted input graphs.


\bibliography{icml}
\bibliographystyle{icml2016}

\clearpage
\appendix
\section{Proofs}
\label{ssec:proofs}
In this section, we indicate the probability of an event $e$ by writing $\Pr(e)$.
\begin{proof}[Proof of Fact~\ref{t:thresold_one_dimension}]\label{p:thresold_one_dimension}
Since the proofs for trollness and unpleasantness are equal, we will prove the statement solely for trollness. 
Without loss of generality, let $\that(i_1) \le \that(i_2) \le \ldots \le \that(i_z)$, where $z$ is the size of training set.

Let $\mhat(T)$ be the function equal to the sum of the two following quantities: 
\begin{itemize}[nosep,leftmargin=*]

\item the number of negative edges outgoing from a node $i \in V$ such that $\that(i) \le T$ and 

\item the number of positive edges outgoing from a node $j \in V$ such that $\that(j) > T$. 

\end{itemize}

The function $\mhat(\cdot)$ represents the number of mistakes made by using a separator for the training set which has only one feature, i.e. a threshold. We will show now that the function $\mhat(T)$ can be always minimised selecting the threshold $T=\frac{1}{2}$ for all training sets and all input labeled graphs. 

Let $h(i_k)$ denote the value of $\mhat(T)$ when $\that(i_{k-1}) < T < \that(i_{k})$.  Note that $\mhat(\cdot)$ is constant on such ranges.  Now, $h(i_{k+1}) = h(i_k) + \dhat_{\mathrm{out}}^-(i_k) - \dhat_{\mathrm{out}}^+(i_k)$. Hence $\mhat(\cdot)$ increases at $T=\that(i_{k})$ if and only if $\dhat_{\mathrm{out}}^-(i_k) > \dhat_{\mathrm{out}}^+(i_k)$, it decreases if and only if $\dhat_{\mathrm{out}}^-(i_k) < \dhat_{\mathrm{out}}^+(i_k)$, and it remains constant if and only if $\dhat_{\mathrm{out}}^-(i_k) = \dhat_{\mathrm{out}}^+(i_k)$. We can then conclude that $\mhat(\cdot)$, as $T$ increases in $[0,1]$, cannot increase and then decrease,  because an increase requires 
$\dhat_{\mathrm{out}}^-(i_k) > \dhat_{\mathrm{out}}^+(i_k)$, which gives $\that(i_k)=\frac{\dhat_{\mathrm{out}}^-(i_k)}{\dhat_{\mathrm{out}}(i_k)} > \frac{1}{2}$, while a decrease requires $\dhat_{\mathrm{out}}^-(i_k) < \dhat_{\mathrm{out}}^+(i_k)$, which gives $\that(i_k)=\frac{\dhat_{\mathrm{out}}^-(i_k)}{\dhat_{\mathrm{out}}(i_k)} < \frac{1}{2}$. 
Hence the function $\mhat(\cdot)$ has always a global minimum point at $\frac{1}{2}$.
 
\end{proof}

\begin{proof}[Proof of Theorem~\ref{t:active_lb}] \label{p:active_lb}[{\em Sketch}]

Without loss of generality we will focus on negative edges. We also assume $Q \neq |E|$. Let be $\mathcal{Y}_K$ be the set of all labelings such that the total number of negative and positive edges are respectively $K$ 
and $|E|-K$. Consider the following randomized strategy to assign a labeling to the edge set of the input graph: we select uniformly at random a labeling $\by$ in $\mathcal{Y}_K$. For each node $i \in V$, we have $\Psi(i) \le d_{\mathrm{out}}^-(i)$, which implies that the statement constraint $\Psi(\by) \le K$ is satisfied. Let $z$ be the number of negative edges present in the test set. 

Assume now that $A$ knows the value of $K$. Observe that, because of the above described randomized strategy used for selecting $\by$, the probability that a test edge is labeled negatively or positively does not depend on the choice of the $Q$ queries made by $A$. However, it is important to take into account that $A$, knowing $K$, knows also $z$ observing the queried labels. This fact is important when $\bE z$ is not an integer. We will now clarify this point with a very simple example. Consider for instance the case in which $|E|=2$, $K=1$ and $Q=1$. In this simple case we have $\bE z=\frac{1}{2}$, independently of the algorithm $A$. However $A$, knowing the value of $K$, will not make any mistake since the information provided in the selection phase will reveal the label of the test edge. More precisely the sign of the test label will be the opposite of the one of the label queried and observed.

In general we have  
$\floor*{\frac{K}{|E|}(|E|-Q)} 
\le \bE z 
\le \ceil*{\frac{K}{|E|}(|E|-Q)}$. 
Hence, for each test edge $i \rightarrow j$ we have that both $ \Pr(y_{i \rightarrow j} = -1)$ and $ \Pr(y_{i \rightarrow j} = +1)$ are lower bounded by $\frac{\floor*{\frac{K}{|E|}(|E|-Q)}}{|E|-Q}$. 

We can therefore conclude that the number of expected mistakes that any algorithm $A$ makes using the above mentioned labeling strategy is never smaller than $\floor*{\frac{K}{|E|}(|E|-Q)} > \frac{K}{|E|}(|E|-Q)-1$.
\end{proof}

\begin{proof}[Proof of Theorem~\ref{t:active_one}] \label{p:active_one} $\alcone$ asks, for each node $i \in V$, the label
$y_{i\rightarrow j}$ of {\em one} edge picked uniformly at random in
$E_{\mathrm{out}}(i)$. The number of labels queried is therefore not larger that $|V|$.
If the label $y_{i\rightarrow j}$ queried for node $i$ is equal to
$y^{\mathrm{min}}_{\mathrm{out}}(i)$, then $\alcone$ makes not more than $d_{\mathrm{out}}(i) - \Psi(i)$
mistakes while predicting the labels of the edges in $E_{\mathrm{out}}(i)$. If instead
$y_{i \rightarrow j} \neq y^{\mathrm{min}}_{\mathrm{out}}(i)$, then the number of mistakes made
on the edge subset $E_{\mathrm{out}}(i)$ is upper bounded by $\Psi(i)$. We have $\Pr \left( y_{i
\rightarrow j} = y^{\mathrm{min}}_{\mathrm{out}}(i) \right)=\frac{\Psi(i)}{d_{\mathrm{out}}(i)}$ and $\Pr
\left( y_{i \rightarrow j} \neq y^{\mathrm{min}}_{\mathrm{out}}(i) \right) =
1-\frac{\Psi(i)}{d_{\mathrm{out}}(i)}$.  In order to simplify the notation, we set now
$P(i)$ equal to $\Pr \left( y_{i \rightarrow j} = y^{\mathrm{min}}_{\mathrm{out}}(i)\right)$
where $j$ is selected uniformly at random in $E_{\mathrm{out}}(i)$ by $\alcone$. The
expected number of mistakes made by $\alcone$ over the random choice of the
query set is then bounded by \begin{align*} \bE m_{\alcone} &=\sum_{i \in V}
\Bigl( P(i) \left(d_{\mathrm{out}}(i) - \Psi(i)\right)+ \left(1-P(i)\right) \Psi(i)
\Bigr) \\ &=\sum_{i \in V} \Bigl( P(i) d_{\mathrm{out}}(i) + \Psi(i)-2P(i)\Psi(i) \Bigr)
  \\ &=\sum_{i \in V} \Biggl(
2\Psi(i)\biggl(1-\frac{\Psi(i)}{d_{\mathrm{out}}(i)}\biggr) \Biggr) \\ &\le 2 \sum_{i
\in V}  \Psi(i) = 2 \Psi\ .  \end{align*} \end{proof}

Observe that the tightness of the above bound depends on the quantities
$\frac{1}{2}-P(i)$ for all $i \in V$. In fact, if these quantities are close to
$0$, then the quantities $1-\frac{\Psi(i)}{d_{\mathrm{out}}(i)}$, bounded above for all
$i \in V$ in the second to last equality, are close to $1$. If most of them are instead close to $\frac{1}{2}$, then $1-\frac{\Psi(i)}{d_{\mathrm{out}}(i)}$ becomes close to $\frac{1}{2}$, making therefore the expected total number of mistakes close to $\Psi$, i.e. close to the expression of the lower bound
provided in Theorem~1.

\begin{proof}[Proof of Theorem~\ref{t:active_log}]\label{p:active_log}
We will use the Chernoff bound to quantify to what extent we can
  estimate in a correct way $y^{\mathrm{min}}_{\mathrm{out}}(i)$ for each node $i \in V$ with
  $\lge$ queries in $E_{\mathrm{out}}(i)$ and we will bound the expected number of
  mistakes made by our prediction rule. 

In order to simplify the notation, for each node $i \in V$ we set now $P(i)$
equal to $\Pr \left( y_{i \rightarrow j} =
y^{\mathrm{min}}_{\mathrm{out}}(i)\right)=\frac{\Psi_{\mathrm{out}}(i)}{d_{\mathrm{out}}(i)}$ where, in
the selection phase, $j$ is chosen uniformly at random in $E_{\mathrm{out}}(i)$ by
$\alclog$. For each node $i \in V$, let $s(i)$ be the total number of labels
observed querying the ones assigned to the edges in $E_{\mathrm{out}}(i)$. Let $s_{\min}(i)$ be the
number of labels observed that are equal to $y^{\mathrm{min}}_{\mathrm{out}}(i)$. 

We predict the non-queried labels of $E_{\mathrm{out}}(i)$ with $y^{\mathrm{min}}_{\mathrm{out}}(i)$
if\footnote{In this analysis we consider the case $s_{\min}(i) = \frac{s(i)}{2}$
when $s$ is even according to the worst possible label choice for the
prediction phase. In this case the label selected to predict the non-queried
ones in $E_{\mathrm{out}}(i)$ is $y^{\mathrm{min}}_{\mathrm{out}}(i)$, i.e. the one maximising the number
of prediction mistakes.} $s_{\min}(i) \ge \frac{s(i)}{2}$. Observe that only if
$s_{\min}(i) \ge \frac{s(i)}{2}$ the algorithm can make $d_{\mathrm{out}}-\Psi(i)$
mistakes, which is the worst prediction case for each node $i \in V$. 

Let $\delta(i)$ be defined as $\frac{1}{2P(i)}-1$ and set $z$ equal to
$\frac{1}{2}-P(i)$. Recall that $s(i)$ is equal to $\lge$. 
Let $\mathcal{E}$ the event $s_{\min}(i)\ge \frac{s(i)}{2}$.
Using the Chernoff
bound, for any $\delta(i) > 0$ we have

\begin{align*} \Pr (\mathcal{E}) &= \Pr \Big(s_{\min}(i)\ge
(1+\delta(i))P(i)s(i)\Big) \\ &\le
\exp\left(-\frac{\delta(i)^2}{2+\delta(i)}P(i)s(i)\right) \\ &=
\exp\left(-\frac{\left(\frac{1}{2P(i)}-1\right)^2}{2+\left(\frac{1}{2P(i)}-1\right)}P(i)s(i)\right)
\\ &= \exp\left(-\frac{(1-2P(i))^2}{2(1+2P(i))}s(i) \right) \\ &=
\exp\left(-\frac{(1-2P(i))^2}{2(1+2P(i))}\lge \right) \\ &\le
\exp\left(-(1-2P(i))^2\lgeone \right) \\ &= \exp\left(-4z^2\lgeone \right) \\
                                         &\le \exp\left(-4z^2(\log (d_{\mathrm{out}}(i)+1))
                                         \right) \\ &= (d_{\mathrm{out}}(i)+1)^{-4z^2}\ .
                                         \end{align*}

Let now $m^{(i)}$ be the number of mistakes made on the labels of the edges
contained in $E_{\mathrm{out}}(i)$.
Observe that when $z=\frac{1}{2}$ we have $\Psi(i)=0$, which implies $m^{(i)}=0$. Taking into consideration that
$d_{\mathrm{out}}(i)-\Psi(i)=\Psi(i)+2z d_{\mathrm{out}}(i)$, which implies $d_{\mathrm{out}}(i)=\frac{2\Psi(i)}{1-2z}$ for $z\neq \frac{1}{2}$, we can then exploit the result obtained above
using the Chernoff bound to write, for all $z \in [0,\frac{1}{2})$,

\begin{align*} \bE m^{(i)} &
\le \left(1-(d_{\mathrm{out}}(i)+1)^{-4z^2}\right)\Psi(i)+\\
&(d_{\mathrm{out}}(i)+1)^{-4z^2}(d_{\mathrm{out}}(i)-\Psi(i))\\
&= \left(1-(d_{\mathrm{out}}(i)+1)^{-4z^2}\right)\Psi(i)+\\
&(d_{\mathrm{out}}(i)+1)^{-4z^2}(\Psi(i)+2z d_{\mathrm{out}}(i))\\
&=\Psi(i)+2zd_{\mathrm{out}}(i)(d_{\mathrm{out}}(i)+1)^{-4z^2}\\
&\le \Psi(i)+2z(d_{\mathrm{out}}(i)+1)^{1-4z^2}\\
&=\Psi(i)+2z\left(\frac{2\Psi(i)}{1-2z}+1\right)^{1-4z^2}\ .
\end{align*}

We will now find an upper bound of the quantity $2z\left(\frac{2\Psi(i)}{1-2z}+1\right)^{1-4z^2}$ for all $z \in [0,\frac{1}{2})$. Notice that when $z$ is close to $0$ we could simplify the analysis taking advantage of the fact that $\frac{2\Psi(i)}{1-2z}=\Theta\left(\Psi(i)\right)$. On the other hand, when $z$ is close to $\frac{1}{2}$, the multiplicative term $2z$ can be bounded by $1$. In order to exploit both these facts, we will split the analysis into two cases: $z \in [0,\frac{1}{4})$ and $z \in [\frac{1}{4},\frac{1}{2})$.

\textbf{Case $z \in [0,\frac{1}{4})$:} We can bound $\frac{2\Psi(i)}{1-2z}$ by $4\Psi(i)$. We have therefore $\bE m^{(i)}\le \Psi(i) + 2z\left(4\Psi(i)+1\right)^{1-4z^2}$. This upper bound is maximised when $z=\left(8\log(4\Psi(i)+1)\right)^{-\frac{1}{2}}$ which is included in $[0,\frac{1}{4})$ for all $\Psi(i) \ge 2$. Bounding the mistake expression $2z\left(4\Psi(i)+1\right)^{1-4z^2}$ with $2z\left(4\Psi(i)+1\right)$ and plugging the $z$'s value $\left(8\log(4\Psi(i)+1)\right)^{-\frac{1}{2}}$ into this upper bound allows us to write 
$\bE m^{(i)}\le \Psi(i) + \frac{8\Psi(i)+2}{\sqrt{8\log(4\Psi(i)+1)}}$\ .
Finally, this bound for $\bE m^{(i)}$ holds even when $\Psi(i)=1$. In fact, setting $\Psi(i)=1$, the original expression term $2z\left(4\Psi(i)+1\right)^{1-4z^2}$ becomes equal to $2z\left(5\right)^{1-4z^2}$ and the upper bound term $\frac{8\Psi(i)+2}{\sqrt{8\log(4\Psi(i)+1)}}$ above obtained becomes equal to $\frac{10}{\sqrt{8\log(5)}}$. It is immediate to verify that $2z5^{1-4z^2} < \frac{10}{\sqrt{8\log(5)}}$ for all $z \in [0, \frac{1}{4})$.

\textbf{Case $z \in [\frac{1}{4},\frac{1}{2})$:} Since we can bound the multiplicative term $2z$ by $1$, we can write $\bE m^{(i)}\le \Psi(i)+\left(\frac{2\Psi(i)}{1-2z}+1\right)^{1-4z^2}$\ . For all $\Psi \ge 1$, in the range $z \in [\frac{1}{4},\frac{1}{2})$ the quantity $\left(\frac{2\Psi(i)}{1-2z}+1\right)^{1-4z^2}$ is always decreasing and it is therefore maximised when $z=\frac{1}{4}$. Hence we have $\bE m^{(i)}\le \Psi(i)+(4\Psi(i)+1)^{\frac{3}{4}}$\ .

\bigskip

We can now compare the two expressions obtained analysing these cases. Bounding again from above the expression $\Psi(i) + \frac{8\Psi(i)+2}{\sqrt{8\log(4\Psi(i)+1)}}$ obtained for the case $z \in [0,\frac{1}{4})$ with $\Psi(i) + \frac{8\Psi(i)+2}{\sqrt{\log(4\Psi(i)+1)}}$ we have that, for all $\Psi(i)>0$, the mistake expression $\Psi(i)+(4\Psi(i)+1)^{\frac{3}{4}}$ obtained for the case $z \in [\frac{1}{4},\frac{1}{2})$ is smaller than $\Psi(i) + \frac{8\Psi(i)+2}{\sqrt{\log(4\Psi(i)+1)}}$ obtained for the case $z \in [0,\frac{1}{4})$. Hence we finally have, for all $\Psi(i)>0$,

\begin{align*} \bE m^{(i)} &
\le \Psi(i) + \frac{8\Psi(i)+2}{\sqrt{\log(4\Psi(i)+1)}} \\
&\le \Psi(i) + \frac{10\Psi(i)}{\sqrt{\log(4\Psi(i)+1)}}\ .
\end{align*}

We will now use the last inequality to write an interpretable bound for $m_{\alclog}$, exploiting Jensen inequality and the concavity of the mistake function. 

Since $\Psi(i) + \frac{10\Psi(i)}{\sqrt{\log(4\Psi(i)+1)}}$ is concave for all $\Psi(i)>0$ and recalling that $\bE m^{(i)}=0$ for all $i$ such that $\Psi(i)=0$, we can conclude that, for all $\Psi(\by)>0$, 

$$\bE m_{\alclog} = 
\sum_{i \in V} \bE m^{(i)} \le 
\Psi(\by)+\mathcal{O}
\left(\frac{\Psi(\by)}{\sqrt{\log\left(4\overline{\Psi}_{0}(\by)+1\right)}}\right)
$$

\bigskip

Finally, the total number of distinct edge labels queried is bounded by
$\sum_{i \in V} \lge \le \sum_{i \in V} 4(\log(d_{\mathrm{out}}(i)+1)+1) 
\le 4|V|+4\sum_{i\in V} \log(d_{\mathrm{out}}(i)+1)$ which, using Jensen inequality together with the concavity of the logarithm function and taking into
account that $\sum_{i \in V} (d_{\mathrm{out}}(i)+1)=|E|+|V|$, can be in turn bounded by
$4|V|+4|V|\log\frac{|E|+|V|}{|V|}$. This expression is equal to $\mathcal{O}\left( |V|\log\left(\frac{|E|}{|V|}+1\right)
\right)$ when $|E|=\Omega(|V|)$\ .
\end{proof}

\section{Additional psychological evidences}

In this section, we elaborate on the psychological findings that support our
approach and its good experimental results on real dataset.
First we emphasize again the importance of the cognitive dissonance theory.
More than half a century after its inception, it is still considered as one of
the fundamental motivations explaining many behaviors, notably because it is a
handy mechanism to 
discover errors in one's system of beliefs \cite{Gawronski2012}. Recently,
this theory have been completed by the action-based model, which provides deeper explanations on the
origin of this phenomenon. Namely, this model posits that \enquote{perceptions
and cognitions serve to activate action tendencies with little or no conscious
deliberation}~\cite{Harmon09}. Dissonance is the result of cognitions being in
conflict, and the organism try to reduce this negative effect by changing
beliefs of attitudes so that all cognitions agree with each others.  This
process is believed to lead in most cases to benefices for the organism,
especially by making the cognitive processes more efficient.

In a classic experiment~\cite{Carlsmith59}, participants were asked to perform
a tedious task (generating a strong negative attitude) and then offered a
monetary incentive of either 8 or 160 dollars\footnote{in present day terms} to
lie to others, pretending the task is actually interesting. At the end of the
experiment, participants were asked to rate the task and they faced some of the
three following facts:
\begin{enumerate}[nosep,leftmargin=*]
  \item The task was tedious
  \item I said to someone else that it is not
  \item I was paid low amount of money to lie (OR)
  \item I was paid high amount of money to lie
\end{enumerate}
Whereas 1. and 2. are in conflict, 4. provides a sufficient justification,
therefore those who receive high amount of money do not experience dissonance.
On the other hand, condition 3. creates dissonance, which is solved by rating
the task more positively than the group in condition 4.

A more direct implication of dissonance avoidance in our context is called the
\enquote{just world belief}. In simple words, people holding such belief assume
that good things happen to good persons and bad things to bad persons. The more
they are convinced of that, the more they tend to blame innocent victims for
their misfortune. That is indeed one way to resolve the inconsistency created by
observing a good, innocent person being a victim (i.e. something bad happening to
her)~\cite{VandenBos09}. In our case, this would justify sending a negative
link to someone who has received a lot of them, no matter whether they are
really nefarious member of the community or not.
On the other hand, the related question of why trolls (i.e. users with a
majority of negative outgoing edges) sometimes output positive edges, cannot be
answered by invoking cognitive dissonance. Indeed dissonance only occurs when
cognitions are in conflict about a single dimension or attribute of an object.
We must thus conclude that some signs are explained by exogenous factors.

Reciprocal edges often having the same sign is also a well documented
psychological effect: reciprocation is \enquote{one of the strongest and most
pervasive social forces in all human cultures, and it helps us build trust with
others and pushes us toward equity in our relationships}~\cite{reciprocity04}.
This natural impulse can be related with the ageless law of talion and its
modern game theory counterpart of \enquote{tit for that}.

The action-based model is consistent with studies of neural activity showing
that dealing with cognitive conflicts increases skin conductance, indicating
arousal of the nervous system~\cite{Hajcak04}. It also activate the anterior
cingulate cortex and the level of activity predicts how much opinions will
change in order to resolve conflicts~\cite{vanVeen2009,Sharot2009}.

The preference for consistency (\pfc{}) that arises from the minimization of
cognitive dissonance can be quantified through a questionnaire on a scale from
1 to 9 and is a good predictor of certain type of behaviors~\cite{PFC95}. A
surrogate of this feature for our users is related to the individual complexity
measure $\Psi_{\mathrm{out}}(i)$ that we defined in
\autoref{sub:complexity_measures}. A user with none or all of his outgoing
edges of the same sign will have a \pfc{} of $0.5$ whereas a user with half
of his outgoing edges positive will be the less consistent with a \pfc{} of
$0$.  On \autoref{fig:pfc} we see that indeed in \wik{} and \epi{}, 75\% of the
nodes with at least one outgoing edges have a consistency larger than $0.4$
even though the situation is more nuanced in \sla{}.

\begin{figure}[]
  \centering
  \begin{subfigure}[b]{1.0\columnwidth}
    \centering
    \includegraphics[width=.9\textwidth]{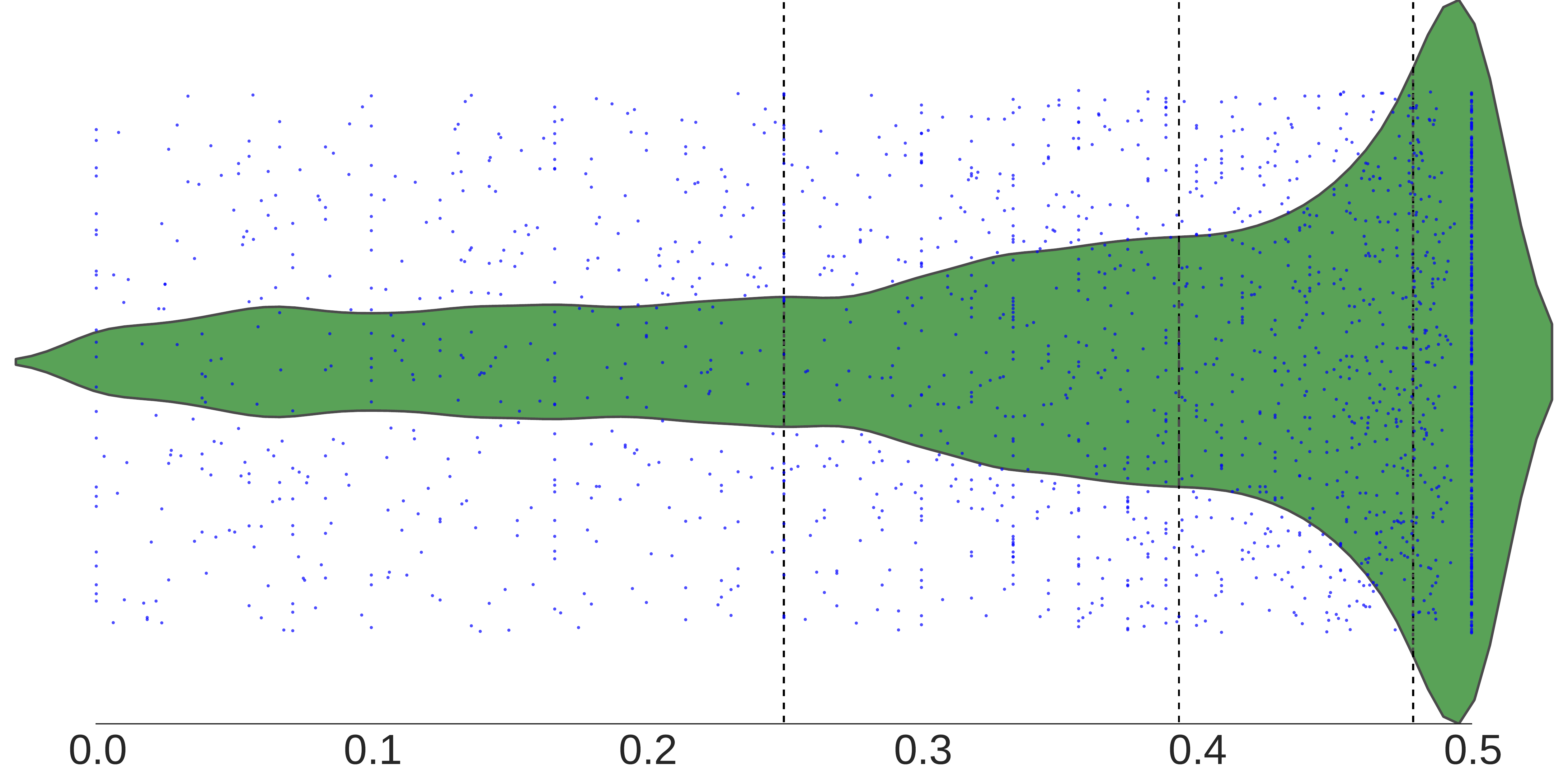}
    \caption{\wik{} \label{fig:wik}}
  \end{subfigure}
  \begin{subfigure}[b]{1.0\columnwidth}
    \centering
    \includegraphics[width=.9\textwidth]{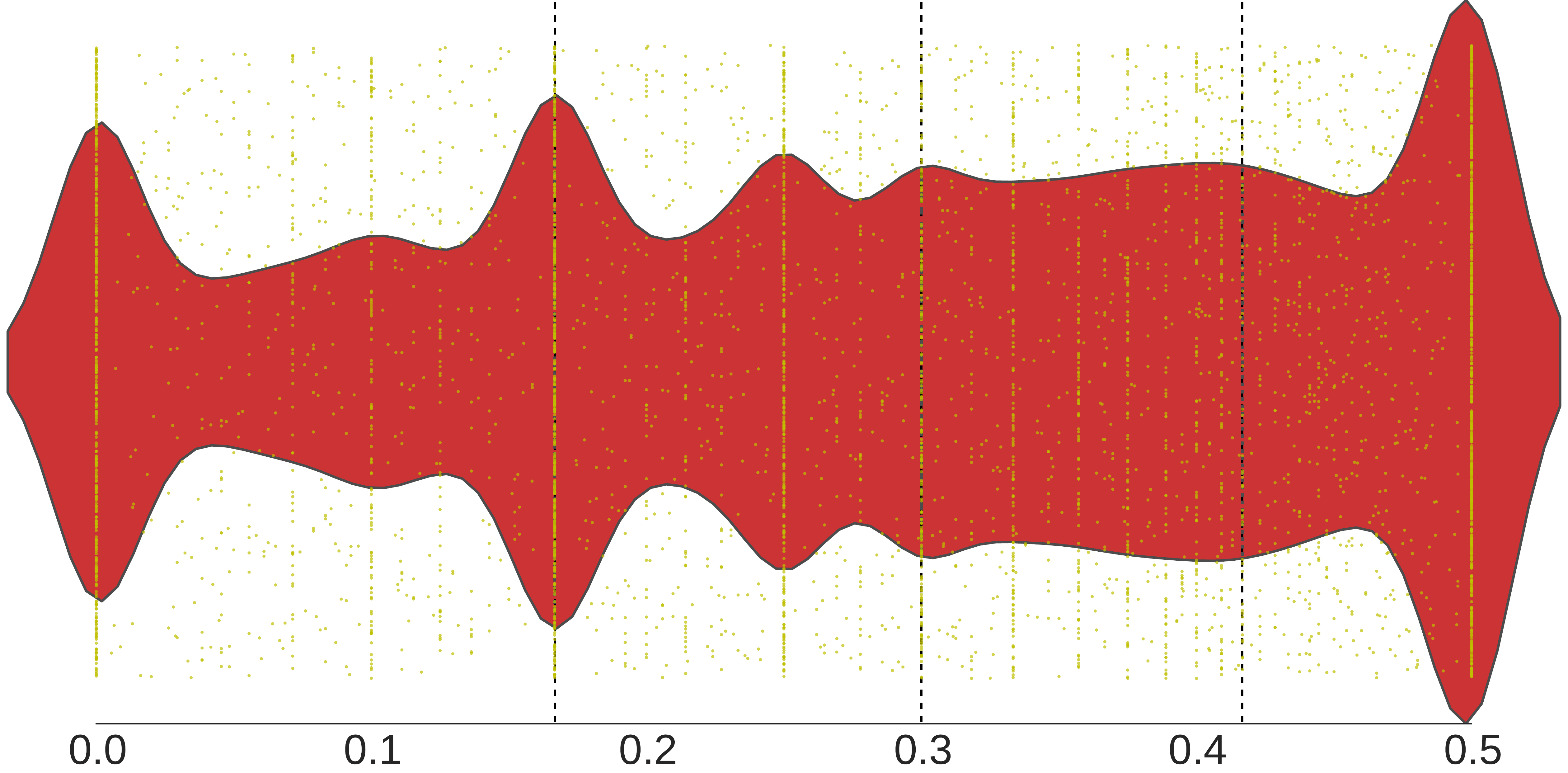}
    \caption{\sla{} \label{fig:sla}}
  \end{subfigure}
  \begin{subfigure}[b]{1.0\columnwidth}
    \centering
    \includegraphics[width=.9\textwidth]{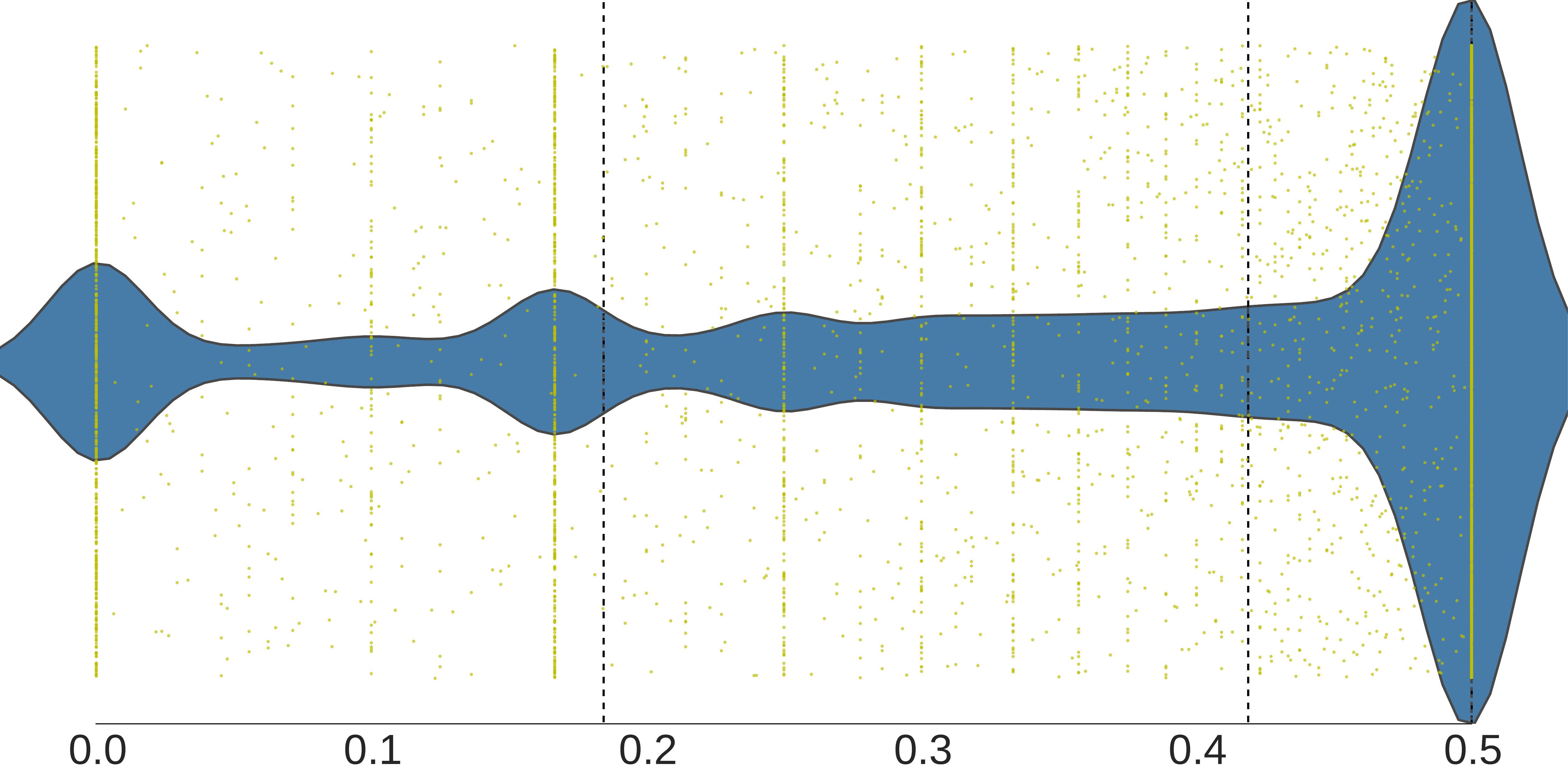}
    \caption{\epi{} \label{fig:epi}}
  \end{subfigure}
  \caption{Distribution of nodes $i$ according to a surrogate of their
  preference for consistency $\nicefrac{1}{2} - \Psi_{\mathrm{out}}(i) =
\left|\nicefrac{1}{2} - t(i)\right|$. Black dashed lines are first, second and
third quartiles from left to right. Each dot is a node and they are scattered
along the vertical dimension to enhance figure legibility. \label{fig:pfc}}
\end{figure}

As people get older, they have witnessed more emotional upset and therefore
have further motivations to reduce emotional instability. A study of 269
individuals between 18 and over 80 years indeed showed a positive relation
between age and high preference for consistency~\cite{agePFC05}. Although we do
not have demographic information about users in our data, this nonetheless
suggest that younger users may have higher $\Psi_{\mathrm{out}}$ and therefore
require more careful estimation of their trollness. However, increasing the performance by driving
the query phase based on side information is outside the scope of this paper
as it would increase the complexity.  
Just as cognitive dissonance affects people differently depending of their age,
the intensity of its effect also depends of cultural background~\cite{east97}.

Another long-discovered psychological effect at play in our situation is the
fact that public commitment tend to be more stronger and more persistent than
those made in private~\cite{publicCom89,publicCialdini98}. It turns out that in
most \ssn{}, the status of a relationship is on display, at least to the other
connected user. This publicity provides further incentive to behave
consistently with one past attitude.

Most of the psychological works cited so far performed their experiments
offline in a lab room yet we apply our method to datasets obtained from
online communities. Hence it is legitimate to wonder whether we can transfer
the conclusions from one domain to another, especially since people tend to be
less inhibited in their online interactions~\cite{Suler04} and since \enquote{
Facebook users tend to be more extroverted and narcissistic, but less
conscientious and socially lonely, than nonusers}~\cite{facebookUser11}.

Fortunately, most people keep their personality when going online, as
evidenced by the fact that online surveys provide similar results to
traditional pen-and-pencil ones~\cite{onlineSurvey04, onlineSurvey07}
including faking answers~\cite{Cyberfaking13}.
Online surveys even \enquote{elicit greater levels of sensitive item self
disclosure}~\cite{onlineDisclosure12}, thereby providing more accurate picture
of respondents. Likewise, communications are mostly unchanged---for instance by
using emoticons, which \enquote{serve the same functions as actual nonverbal
behavior}~\cite{Emoticons08}---or enhanced through virtual games, which
\enquote{allow players to express themselves in ways they may not
feel comfortable doing in real life because of their appearance, gender,
sexuality, and/or age.}~\cite{MMO07}. Measuring the Big Five personality
factors of 122 student participants, \citet{sameOnline12} even concluded that
\enquote{the emotional stability increases on the computer and the Internet}.
Finally, online structures closely mirror offline one \cite{mirror15}.

\section{Additional results}
Full result  of accuracy (\autoref{tab:batch_postprocessing}) and MCC
(\autoref{tab:batch_postprocessing_mcc}) in the batch settings follow on the
next page, comparing our method (with and without reciprocal edges heuristic)
against competitors.
\begin{table*}[t]\centering
  \setlength{\tabcolsep}{3.5pt}
  \footnotesize
  \caption{Full numerical accuracy results in the batch settings.
    The starred version (${}^\star$) are the ones where we use the
    reciprocal edge heuristics described at the end of \autoref{sec:results}.
    While there is no significant difference in \wik{}, for the two others
    dataset where recriprocal edges fraction is respectively 8 and 15\%, the
    accuracy is up by 0.4 to 1.2\%, which makes our methods outperform all
    competitors in all cases.
  \label{tab:batch_postprocessing}}
  \begin{tabular}{lcc|cc|cc|ccc}
    \toprule
    $\frac{|E_{\mathrm{train}}|}{|E|}$ & \blc{}$(t)^\star$ & \blc{}$(t)$ & \blc{}$(t,u)^\star$ & \blc{}$(t,u)$ & \comppp{}$^\star$ & \comppp{} & \comptriads{} & \complowrank{} & \compasym{} \\
    \midrule
    & \multicolumn{9}{c}{\wik{}} \\
    \cmidrule{2-10}
    15\% & $82.97 \pm .12$  & $82.93 \pm .12$  & $83.68 \pm .26$  & $83.99 \pm .08$  & $83.28 \pm .15$  & $83.39 \pm .21$  & $80.73 \pm .28$  & $81.07 \pm .18$  & $78.50 \pm .25$ \\
    30\% & $84.42 \pm .14$  & $84.48 \pm .10$  & $85.48 \pm .14$  & $85.45 \pm .16$  & $85.07 \pm .14$  & $85.08 \pm .09$  & $82.27 \pm .22$  & $82.62 \pm .14$  & $80.89 \pm .23$ \\
    45\% & $84.89 \pm .14$  & $84.88 \pm .10$  & $86.11 \pm .14$  & $86.08 \pm .10$  & $85.77 \pm .14$  & $85.79 \pm .12$  & $83.43 \pm .21$  & $83.53 \pm .21$  & $81.81 \pm .16$ \\
    60\% & $84.97 \pm .11$  & $85.20 \pm .16$  & $86.43 \pm .18$  & $86.65 \pm .10$  & $86.13 \pm .16$  & $86.26 \pm .11$  & $84.33 \pm .18$  & $84.33 \pm .28$  & $82.14 \pm .17$ \\
    75\% & $85.06 \pm .15$  & $85.20 \pm .20$  & $86.62 \pm .18$  & $86.68 \pm .19$  & $86.38 \pm .18$  & $86.41 \pm .12$  & $85.07 \pm .22$  & $84.74 \pm .22$  & $82.50 \pm .26$ \\
    90\% & $85.30 \pm .34$  & $85.51 \pm .10$  & $86.90 \pm .30$  & $87.16 \pm .10$  & $86.68 \pm .33$  & $86.89 \pm .10$  & $85.49 \pm .39$  & $85.39 \pm .31$  & $82.69 \pm .26$ \\ 
    \midrule
    & \multicolumn{9}{c}{\sla{}} \\
    \cmidrule{2-10}
    15\% & $82.71 \pm .06$  & $82.50 \pm .02$  & $80.06 \pm .09$  & $79.66 \pm .20$  & $79.62 \pm .07$  & $79.12 \pm .09$  & $81.64 \pm .34$  & $81.20 \pm .10$  & $68.60 \pm .19$ \\
    30\% & $83.66 \pm .05$  & $83.24 \pm .02$  & $83.13 \pm .06$  & $82.50 \pm .04$  & $82.44 \pm .06$  & $81.70 \pm .01$  & $83.75 \pm .22$  & $82.54 \pm .23$  & $72.70 \pm .13$ \\
    45\% & $84.26 \pm .05$  & $83.67 \pm .05$  & $84.61 \pm .04$  & $83.81 \pm .07$  & $84.01 \pm .05$  & $83.08 \pm .07$  & $84.58 \pm .16$  & $83.69 \pm .24$  & $74.63 \pm .08$ \\
    60\% & $84.70 \pm .05$  & $84.00 \pm .05$  & $85.56 \pm .07$  & $84.63 \pm .05$  & $84.98 \pm .08$  & $83.93 \pm .06$  & $85.06 \pm .12$  & $84.64 \pm .16$  & $75.79 \pm .13$ \\
    75\% & $85.05 \pm .09$  & $84.13 \pm .08$  & $86.21 \pm .07$  & $85.11 \pm .11$  & $85.70 \pm .08$  & $84.41 \pm .07$  & $85.48 \pm .09$  & $85.24 \pm .21$  & $76.64 \pm .10$ \\
    90\% & $85.36 \pm .16$  & $84.28 \pm .19$  & $86.72 \pm .16$  & $85.45 \pm .22$  & $86.21 \pm .18$  & $84.82 \pm .19$  & $85.83 \pm .16$  & $85.57 \pm .17$  & $77.16 \pm .24$ \\
    \midrule
    & \multicolumn{9}{c}{\epi{}} \\
    \cmidrule{2-10}
    15\% & $89.19 \pm .05$  & $88.91 \pm .03$  & $89.05 \pm .06$  & $88.70 \pm .05$  & $88.70 \pm .05$  & $88.26 \pm .07$  & $88.30 \pm .52$  & $88.50 \pm .08$  & $79.56 \pm 0.18$ \\
    30\% & $89.76 \pm .04$  & $89.27 \pm .02$  & $91.07 \pm .05$  & $90.67 \pm .04$  & $90.75 \pm .05$  & $90.19 \pm .06$  & $90.20 \pm .68$  & $89.88 \pm .10$  & $79.64 \pm 1.07$ \\
    45\% & $90.14 \pm .03$  & $89.42 \pm .02$  & $92.01 \pm .05$  & $91.50 \pm .04$  & $91.71 \pm .04$  & $90.93 \pm .02$  & $90.85 \pm .56$  & $90.52 \pm .09$  & $61.06 \pm 3.24$ \\
    60\% & $90.41 \pm .03$  & $89.46 \pm .02$  & $92.53 \pm .04$  & $91.93 \pm .03$  & $92.34 \pm .03$  & $91.37 \pm .06$  & $91.57 \pm .45$  & $90.78 \pm .08$  & $64.92 \pm 3.79$ \\
    75\% & $90.69 \pm .05$  & $89.57 \pm .03$  & $92.94 \pm .05$  & $92.28 \pm .05$  & $92.84 \pm .05$  & $91.73 \pm .03$  & $91.75 \pm .27$  & $90.90 \pm .07$  & $59.45 \pm 3.01$ \\
    90\% & $90.88 \pm .09$  & $89.57 \pm .02$  & $93.25 \pm .06$  & $92.49 \pm .02$  & $93.19 \pm .07$  & $91.92 \pm .04$  & $91.93 \pm .19$  & $90.85 \pm .11$  & $52.52 \pm 0.56$ \\
    \bottomrule
  \end{tabular}
\end{table*}

\begin{table*}[b]\centering
  \setlength{\tabcolsep}{3.3pt}
  \footnotesize
  \caption{Full numerical MCC results in the batch settings, where the same
    conclusion holds to an even larger extent (almost up to 4\% on the 90\%
    sample of \epi{})
  \label{tab:batch_postprocessing_mcc}}
  \begin{tabular}{lcc|cc|cc|ccc}
    \toprule
    $\frac{|E_{\mathrm{train}}|}{|E|}$ & \blc{}$(t)^\star$ & \blc{}$(t)$ & \blc{}$(t,u)^\star$ & \blc{}$(t,u)$ & \comppp{}$^\star$ & \comppp{} & \comptriads{} & \complowrank{} & \compasym{} \\
    \midrule
    & \multicolumn{9}{c}{\wik{}} \\
    \cmidrule{2-10}
    15\% & $43.43 \pm .44$  & $43.45 \pm .42$  & $49.65 \pm .39$  & $49.58 \pm .60$  & $50.66 \pm .47$  & $50.88 \pm .39$  & $34.29 \pm .78$  & $35.31 \pm .68$  & $33.95 \pm .85$ \\
    30\% & $48.23 \pm .47$  & $48.37 \pm .30$  & $54.49 \pm .65$  & $54.50 \pm .13$  & $55.52 \pm .37$  & $55.52 \pm .16$  & $42.80 \pm .76$  & $42.20 \pm .51$  & $39.21 \pm .59$ \\
    45\% & $49.81 \pm .45$  & $49.70 \pm .12$  & $56.49 \pm .47$  & $56.87 \pm .41$  & $57.51 \pm .38$  & $57.54 \pm .32$  & $47.78 \pm .63$  & $46.10 \pm .71$  & $41.49 \pm .41$ \\
    60\% & $49.95 \pm .50$  & $50.48 \pm .42$  & $57.54 \pm .54$  & $57.79 \pm .40$  & $58.39 \pm .47$  & $58.73 \pm .17$  & $50.96 \pm .69$  & $49.52 \pm .99$  & $42.60 \pm .57$ \\
    75\% & $50.46 \pm .44$  & $50.91 \pm .99$  & $58.37 \pm .56$  & $58.49 \pm .71$  & $59.16 \pm .52$  & $59.30 \pm .63$  & $53.75 \pm .72$  & $51.47 \pm .59$  & $43.96 \pm .59$ \\
    90\% & $50.96 \pm .99$  & $51.39 \pm .72$  & $58.89 \pm .81$  & $59.76 \pm .94$  & $59.84 \pm .97$  & $60.46 \pm .67$  & $54.98 \pm .99$  & $53.81 \pm .98$  & $44.42 \pm .82$ \\
    \midrule
    & \multicolumn{9}{c}{\sla{}} \\
    \cmidrule{2-10}
    15\% & $48.61 \pm .20$  & $48.03 \pm .17$  & $46.34 \pm .23$  & $45.51 \pm .39$  & $47.34 \pm .17$  & $46.44 \pm .09$  & $45.76 \pm .68$  & $41.33 \pm .35$  & $20.54 \pm .43$ \\
    30\% & $51.11 \pm .17$  & $49.87 \pm .05$  & $52.24 \pm .26$  & $50.53 \pm .23$  & $53.01 \pm .13$  & $51.44 \pm .08$  & $52.10 \pm .60$  & $46.67 \pm .78$  & $25.92 \pm .38$ \\
    45\% & $52.74 \pm .18$  & $50.91 \pm .08$  & $55.53 \pm .25$  & $53.36 \pm .18$  & $56.34 \pm .11$  & $54.27 \pm .14$  & $54.64 \pm .38$  & $51.17 \pm .77$  & $29.13 \pm .27$ \\
    60\% & $53.90 \pm .18$  & $51.71 \pm .11$  & $57.72 \pm .30$  & $55.35 \pm .12$  & $58.54 \pm .21$  & $56.10 \pm .15$  & $56.19 \pm .30$  & $54.97 \pm .53$  & $31.41 \pm .31$ \\
    75\% & $54.87 \pm .19$  & $52.00 \pm .15$  & $59.48 \pm .27$  & $56.45 \pm .29$  & $60.21 \pm .22$  & $57.09 \pm .15$  & $57.43 \pm .21$  & $57.45 \pm .63$  & $33.17 \pm .35$ \\
    90\% & $55.71 \pm .48$  & $52.18 \pm .66$  & $60.88 \pm .47$  & $57.24 \pm .71$  & $61.41 \pm .49$  & $58.06 \pm .62$  & $58.44 \pm .46$  & $59.05 \pm .49$  & $34.56 \pm .60$ \\
    \midrule
    & \multicolumn{9}{c}{\epi{}} \\
    \cmidrule{2-10}
    15\% & $53.04 \pm .26$  & $52.14 \pm .21$  & $58.56 \pm .44$  & $57.63 \pm .53$  & $60.16 \pm .14$  & $59.20 \pm .12$  & $54.79 \pm 1.31$  & $48.19 \pm .47$  & $37.55 \pm 0.37$ \\
    30\% & $54.94 \pm .21$  & $53.08 \pm .11$  & $64.57 \pm .31$  & $63.32 \pm .27$  & $65.60 \pm .14$  & $64.14 \pm .11$  & $60.94 \pm 1.81$  & $57.50 \pm .45$  & $39.52 \pm 1.82$ \\
    45\% & $56.37 \pm .18$  & $53.67 \pm .04$  & $67.49 \pm .28$  & $65.76 \pm .16$  & $68.33 \pm .13$  & $66.14 \pm .06$  & $63.40 \pm 1.60$  & $62.25 \pm .39$  & $29.05 \pm 2.43$ \\
    60\% & $57.48 \pm .14$  & $53.84 \pm .25$  & $69.28 \pm .23$  & $67.23 \pm .13$  & $70.25 \pm .12$  & $67.46 \pm .19$  & $66.04 \pm 1.36$  & $65.07 \pm .36$  & $32.79 \pm 2.84$ \\
    75\% & $58.52 \pm .22$  & $54.21 \pm .23$  & $70.72 \pm .23$  & $68.37 \pm .16$  & $71.78 \pm .20$  & $68.56 \pm .20$  & $66.85 \pm 0.83$  & $67.02 \pm .30$  & $31.46 \pm 1.95$ \\
    90\% & $59.34 \pm .36$  & $54.19 \pm .09$  & $71.96 \pm .30$  & $69.18 \pm .15$  & $72.94 \pm .26$  & $69.04 \pm .06$  & $67.73 \pm 0.67$  & $68.47 \pm .41$  & $28.73 \pm 0.39$ \\
    \bottomrule
  \end{tabular}
\end{table*}

\todos
\end{document}